\documentclass{article}

\usepackage[utf8]{inputenc} %
\usepackage[T1]{fontenc}    %
\usepackage{hyperref}       %
\usepackage{url}            %
\usepackage{booktabs}       %
\usepackage{amsfonts}       %
\usepackage{nicefrac}       %
\usepackage{microtype}      %
\usepackage{tikz}
\usetikzlibrary{decorations.pathreplacing,shapes.misc}
\usepackage{relsize}
\usetikzlibrary{calc}
\usepackage{times}
\usepackage{mathrsfs}
\usepackage{amssymb}
\usepackage{latexsym}
\usepackage{amsmath,amsthm}
\usepackage{amsfonts}
\usepackage{url}
\usepackage{geometry}
\usepackage{bm}
\usepackage{enumitem}
\usepackage{stmaryrd}  %
\newcommand{\pred}[1]{\boldsymbol{{1}}[#1]}
\usepackage{algorithm}
\usepackage{algpseudocode}
\usepackage{todonotes}
\usepackage{hyperref}
\newcommand{\C}{{\cal C}}

\renewcommand{\P}{\mathbb{P}}
\newcommand{\E}{\mathbb{E}}

\newcommand{\SP}{{\cal V}}
\renewcommand{\sp}{V} %
\newcommand{\teta}{\tilde\eta}
\newcommand{\tr}{\tilde r}
\newcommand{\X}{{\cal X}}
\newcommand{\Y}{{\cal Y}}

\newcommand{\N}{\mathbb{N}}
\newcommand{\R}{\mathbb{R}}

\newcommand{\beq}{\begin{eqnarray*}}
\newcommand{\eeq}{\end{eqnarray*}}
\newcommand{\beqn}{\begin{eqnarray}}
\newcommand{\eeqn}{\end{eqnarray}}

\newcommand{\abs}[1]{\left| #1 \right|}
\newcommand{\hide}[1]{}
\newcommand{\set}[1]{\left\{ #1 \right\}}
\newcommand{\eps}{\varepsilon}

\newcommand{\err}{\mathrm{err}}
\newcommand{\serr}{\widehat{\err}}
\newcommand{\oo}[1]{\frac{1}{#1}}
\newcommand{\argmin}{\mathop{\mathrm{argmin}}}
\newcommand{\argmax}{\mathop{\mathrm{argmax}}}

\newcommand{\dist}{\rho} %
\newcommand{\g}{\gamma}
\newcommand{\ddim}{\operatorname{ddim}}

\newcommand{\diam}{\operatorname{diam}}

\newcommand{\gn}{\, | \,}
\newtheorem{theorem}{Theorem}
\newtheorem{example}{Example}

\newtheorem{lemma}[theorem]{Lemma}
\newtheorem{remark}{Remark}

\newtheorem{proposition}[theorem]{Proposition}
\newcommand{\bepf}{\begin{proof}}
\newcommand{\enpf}{\end{proof}}

\newcommand{\tS}{S'}

\newcommand{\tY}{Y'}
\newcommand{\tbY}{{\bm Y}'}

\newcommand{\nn}{\text{nn}}

\renewcommand{\k}{m}

\renewcommand{\a}{\alpha}
\newcommand{\gnet}{{\bm X}}

\newcommand{\Ng}{\mathcal{N}_\g}
\newcommand{\UB}{\textrm{UB}}
\newcommand{\missmass}{L}

\newcommand{\Vor}{\mathcal{V}}

\newcommand{\cB}{\bar B} %

\newcommand{\appref}[1]{Appendix \ref{ap:#1}}

\newcommand{\nats}{\mathbb{N}}
\newcommand{\lemref}[1]{Lemma \ref{lem:#1}}
\newcommand{\thmref}[1]{Theorem \ref{thm:#1}}
\renewcommand{\algref}[1]{Alg.~\ref{alg:#1}}
\renewcommand{\eqref}[1]{Eq.~(\ref{eq:#1})}
\newcommand{\secref}[1]{Sec.~\ref{sec:#1}}
\newcommand{\exref}[1]{Example~\ref{ex:#1}}
\newcommand{\propref}[1]{Proposition~\ref{prop:#1}}

\newcommand{\fins}{Z_{0}}
\newcommand{\infs}{Z_{\infty}}

\newcommand{\anc}{\tilde x}
\newcommand{\ancb}{\tilde q}
\newcommand{\ca}[2]{{#1 \wedge #2}}
\newcommand{\seq}[1]{\g_{#1}}

\newcommand{\bmu}{{\bar{\mu}}}

\newcommand{\VIIa}{\V^{II_a}}
\newcommand{\VIIb}{\V^{II_b}}
\newcommand{\VIa}{\V^{I_a}}
\newcommand{\VIb}{\V^{I_b}}

\newcommand{\V}{\mathcal V}
\newcommand{\Tu}{\overline T}

\newcommand{\prefix}{\prec}
\newcommand{\algname}{\textup{{\textsf{KSU}}}}
\newcommand{\gkn}{\textup{{\textsf{GKN}}}}
\newcommand{\gkk}{\textup{{\textsf{GKK}}}}

\newcommand{\citet}[1]{\cite{#1}}
\newcommand{\citep}[1]{\cite{#1}}

\newcommand{\mydash}{\text{-}}

\newcommand{\ceil}[1]{\ensuremath{\left\lceil#1\right\rceil}}

\newcommand{\cl}{\operatorname{cl}}
\renewcommand{\gets}{:=}

\title{Nearest-Neighbor Sample Compression:\\ Efficiency, Consistency, Infinite Dimensions}

\hide{
Aryeh Kontorovich
karyeh@cs.bgu.ac.il
Ben-Gurion University

Sivan Sabato
sabatos@cs.bgu.ac.il
Ben-Gurion University

Roi Weiss
weiss.roi@gmail.com
Weizmann Institute of Science
}

\author{
  Aryeh Kontorovich\thanks{Ben-Gurion University of the Negev}
\and
Sivan Sabato$^*$
\and
Roi Weiss\thanks{Weizmann Institute of Science}
}

\begin{document}

\maketitle

\begin{abstract}
We examine the Bayes-consistency of a recently proposed 1-nearest-neighbor-based multiclass learning algorithm.
This algorithm is derived from sample compression bounds and enjoys the statistical advantages of tight, fully empirical generalization bounds, as well as the algorithmic advantages of
a faster
runtime and memory savings.
We prove that this algorithm is strongly Bayes-consistent in metric spaces with finite doubling dimension --- the first consistency result for an efficient nearest-neighbor sample compression scheme.
Rather surprisingly, we discover that this algorithm continues to be Bayes-consistent
even in a certain infinite-dimensional setting, in which the basic measure-theoretic
conditions on which classic consistency proofs hinge are violated.
This is all the more surprising, since it is known that $k$-NN is not Bayes-consistent in this setting. We pose several challenging open problems for future research.
\end{abstract}

\section{Introduction}
\label{sec:introduction}
This paper deals with Nearest-Neighbor (NN) learning algorithms in metric spaces. Initiated by Fix and Hodges in 1951 \cite{FH1989}, this seemingly
naive
learning paradigm remains competitive against more sophisticated methods \cite{DBLP:conf/cvpr/BoimanSI08,DBLP:journals/jmlr/WeinbergerS09} and, in its celebrated $k$-NN version, has been placed on a solid theoretical foundation \cite{CoverHart67,stone1977,MR780746,MR877849}.

Although the classic 1-NN is well known to be inconsistent in general, in recent years a series of papers has presented variations on the theme of a {\em regularized} $1$-NN classifier, as an alternative to the Bayes-consistent $k$-NN.
Gottlieb et al. \cite{DBLP:journals/tit/GottliebKK14+colt} showed that approximate nearest neighbor search can act as a regularizer, actually improving generalization performance rather than just injecting noise.
In a follow-up work, \cite{kontorovich2014bayes} showed that applying Structural Risk Minimization to (essentially) the margin-regularized data-dependent bound in \cite{DBLP:journals/tit/GottliebKK14+colt} yields a strongly Bayes-consistent 1-NN classifier.
A further development has seen margin-based regularization analyzed through the lens of sample compression: a near-optimal nearest neighbor condensing algorithm was presented \cite{gknips14} and later extended to cover semimetric spaces \cite{gkn-jmlr17+aistats}; an activized version also appeared \cite{kontorovichsabatourner16}.
As detailed in \cite{kontorovich2014bayes}, margin-regularized 1-NN methods enjoy a number of statistical and computational advantages over the traditional $k$-NN classifier.
Salient among these are explicit data-dependent generalization bounds, and considerable runtime and memory savings.
Sample compression affords additional advantages, in the form of tighter generalization bounds and increased efficiency in time and space.

In this work we study the Bayes-consistency of a compression-based $1$-NN multiclass learning algorithm, in both finite-dimensional and infinite-dimensional metric spaces.
The algorithm is essentially the passive component of the active learner proposed by  Kontorovich, Sabato, and Urner in \cite{kontorovichsabatourner16}, and we refer to it in the sequel as \algname;
for completeness, we present it here in full (\algref{simple}).
We show that in finite-dimensional metric spaces, \algname\ is both computationally efficient and Bayes-consistent.
This is the first compression-based multiclass 1-NN algorithm proven to possess both of these properties.
We further exhibit a surprising phenomenon in infinite-dimensional spaces, where we construct a distribution for which \algname\ is Bayes-consistent while $k$-NN is not.

\paragraph{Main results.} Our main contributions consist of analyzing the performance of \algname\ in finite and infinite dimensional settings, and comparing it to the classical $k$-NN learner.
Our key findings are summarized below.
\begin{itemize}
\item In \thmref{comp-consist}, we show that \algname\ is computationally efficient and strongly Bayes-consistent in metric spaces with a finite doubling dimension.
This is the first (strong or otherwise) Bayes-consistency result for an efficient sample compression scheme for a multiclass (or even binary)\footnote{ An efficient sample compression algorithm was given in \cite{gknips14} for the binary case, but no Bayes-consistency guarantee is known for it.}
$1$-NN algorithm.
This result should be contrasted with the one in \cite{kontorovich2014bayes}, where margin-based regularization was employed, but not compression; the proof techniques from \cite{kontorovich2014bayes} do not carry over to the compression-based scheme. Instead, novel arguments are required, as we discuss below.
The new sample compression technique provides a Bayes-consistency proof for multiple (even countably many) labels; this is contrasted with the multiclass 1-NN algorithm in \cite{kontorovich2014maximum}, which is not compression-based, and requires solving a minimum vertex cover problem, thereby imposing a $2$-approximation factor whenever there are more than two labels. 

\item In \thmref{ksubayespreiss}, we make the surprising discovery   that \algname\ continues to be Bayes-consistent in a certain infinite-dimensional setting, even though this setting violates the basic measure-theoretic conditions on which classic consistency proofs hinge, including \thmref{comp-consist}.
This is all the more surprising, since it is known that $k$-NN is not Bayes-consistent for this construction \cite{cerou2006nearest}.
We are currently unaware of any separable\footnote{C{\'e}rou and Guyader    \cite{cerou2006nearest} gave a simple example of a nonseparable metric on which all known nearest-neighbor methods, including $k$-NN and \algname, obviously fail.}
metric probability space on which \algname\ fails to be Bayes-consistent; this is posed as an intriguing open problem.
\end{itemize}  

Our results indicate that in finite dimensions, an efficient, compression-based, Bayes-consistent multiclass 1-NN algorithm exists, and hence can be offered as an alternative to $k$-NN, which is well known to be Bayes-consistent in finite dimensions \cite{DBLP:journals/pami/Devroye81,shwartz2014understanding}.
In contrast, in infinite dimensions, our results show that the condition characterizing the Bayes-consistency of $k$-NN does not extend to all NN algorithms.
It is an open problem to characterize the necessary and sufficient conditions for the existence of a Bayes-consistent NN-based algorithm in infinite dimensions.

\paragraph{Related work.}
Following the pioneering work of \citet{CoverHart67} on nearest-neighbor classification, it was shown by \citet{MR780746,MR877849,devroye2013probabilistic} that the $k$-NN classifier is strongly Bayes consistent in $\R^d$.
These results made extensive use of the Euclidean structure of $\R^d$, but in \citet{shwartz2014understanding} a weak Bayes-consistency result was shown for metric spaces with a bounded diameter and a bounded doubling dimension, and additional distributional smoothness assumptions.
More recently, some of the classic results on $k$-NN risk decay rates were refined by \citet{DBLP:journals/corr/ChaudhuriD14} in an analysis that captures the interplay between the metric and the sampling distribution.
The worst-case rates have an exponential dependence on the dimension (i.e., the so-called
{\em curse of dimensionality}), and Pestov
\cite{MR1741506,MR3061714}
examines this phenomenon closely under various distributional and structural assumptions.

Consistency of NN-type algorithms in more general (and in particular infinite-dimensional)
metric spaces was discussed in
\cite{MR2327897,MR2235289,MR2654492,cerou2006nearest,MR1366756}.
In \cite{MR2327897,cerou2006nearest},
characterizations of Bayes-consistency were given in terms of
Besicovitch-type conditions (see \eqref{preissprop}).
In \cite{MR2327897}, a generalized
``moving window'' classification rule is used
and additional regularity conditions on the regression function are imposed.
The {\em filtering} technique (i.e., taking the first $d$ coordinates
in some basis representation) was shown to be universally consistent in
\cite{MR2235289}.
However, that algorithm suffers from the cost of cross-validating over {\em both}
the dimension $d$ and number of neighbors $k$.
Also, the technique is only applicable in Hilbert spaces
(as opposed to more general metric spaces) and
provides only asymptotic consistency,
without finite-sample bounds
such as those provided by
\algname. 
The insight of \cite{MR2235289}
is extended to the more general Banach spaces in
\cite{MR2654492} under various regularity assumptions.

None of the aforementioned generalization results for NN-based techniques
are in the form of fully empirical,
explicitly computable sample-dependent error bounds.
Rather, they are stated in terms of the unknown Bayes-optimal rate,
and some involve additional parameters quantifying the well-behavedness of the unknown distribution
(see \cite{kontorovich2014bayes} for a detailed discussion).
As such, these guarantees do not enable a practitioner to compute a numerical generalization error estimate for a given training sample, much less allow for a data-dependent selection of $k$, which must be tuned via cross-validation.
The asymptotic expansions in \citet{MR1635410,335893,hall2005,samworth2012}
likewise do not provide a computable finite-sample bound.
The quest for such bounds was a key motivation behind the series of works \cite{DBLP:journals/tit/GottliebKK14+colt,kontorovich2014maximum,gknips14}, of which \algname\ \cite{kontorovichsabatourner16} is the latest development.

The work of Devroye et al.~\cite[Theorem 21.2]{devroye2013probabilistic} has implications for $1$-NN classifiers in $\R^d$ that are defined based on data-dependent majority-vote partitions of the space.
It is shown that under some conditions, a fixed mapping from each sample size to a data-dependent partition rule induces a strongly Bayes-consistent algorithm. This result requires the partition rule to have a bounded VC dimension, and since this rule must be fixed in advance, the algorithm is not fully adaptive.
Theorem 19.3 ibid. proves weak consistency for an inefficient compression-based algorithm, which selects among all the possible compression sets of a certain size, and maintains a certain rate of compression relative to the sample size. 
The generalizing power of sample compression was independently discovered by \citet{warmuth86}, and later elaborated upon by \citet{graepel2005pac}.
In the context of NN classification, \citet{devroye2013probabilistic} lists various condensing heuristics (which have no known performance guarantees) and leaves open the algorithmic question of how to minimize the empirical risk over all subsets of a given size.

The first compression-based 1-NN algorithm with provable optimality guarantees was given in \cite{gknips14}; it was based on constructing $\gamma$-nets in spaces with a finite doubling dimension.
The compression size of this construction was shown to be nearly unimprovable by an efficient algorithm unless P=NP.
With $\g$-nets as its algorithmic engine, \algname\ inherits this near-optimality.
The compression-based $1$-NN paradigm was later extended to semimetrics in \cite{gkn-jmlr17+aistats}, where it was shown to survive violations of the triangle inequality, while the hierarchy-based search methods that have become standard for metric spaces (such as \cite{BKL06,DBLP:journals/tit/GottliebKK14+colt} and related approaches) all break down.

It was shown in \cite{kontorovich2014bayes} that a margin-regularized $1$-NN learner (essentially, the one proposed in \cite{DBLP:journals/tit/GottliebKK14+colt}, which, unlike \cite{gknips14}, did not involve sample compression) becomes strongly Bayes-consistent when the margin is chosen optimally in an explicitly prescribed sample-dependent fashion.
The margin-based technique developed in \cite{DBLP:journals/tit/GottliebKK14+colt} for the binary case was extended to multiclass in \cite{kontorovich2014maximum}.
Since the algorithm relied on computing a minimum vertex cover, it was not possible to make it both computationally efficient and Bayes-consistent when the number of lables exceeds two.
An additional improvement over \cite{kontorovich2014maximum} is that the generalization bounds presented there had an explicit (logarithmic) dependence on the number of labels, while our compression scheme extends seamlessly to countable label spaces.

\paragraph{Paper outline.} 
After fixing the notation and setup in \secref{setting}, in \secref{compression_scheme} we present \algname, the compression-based 1-NN algorithm we analyze in this work.
\secref{main_results} discusses our main contributions regarding \algname, together with some open problems. High-level proof sketches are given in \secref{consistency_proof} for the finite-dimensional case, and \secref{infinite} for the infinite-dimensional case. Full detailed proofs are found in the appendices.

\section{Setting and Notation}
\label{sec:setting}
Our instance space is the metric space $(\X,\dist)$, where $\X$ is the instance domain and $\rho$ is the metric.
(See \appref{metric_space_basics} for relevant background on metric measure spaces.)
We consider a countable label space $\Y$.
The unknown sampling distribution is a probability measure $\bmu$ over $\X\times\Y$, with marginal $\mu$ over $\X$.
Denote by $(X,Y) \sim \bmu$ a pair drawn according to $\bmu$. 
The generalization error of a classifier $f:\X \rightarrow \Y$ is given by $\err_\bmu(f) := \P_\bmu (Y \neq f(X))$, and its empirical error  with respect to a labeled set $\tS\subseteq \X \times \Y$ is given by $\serr(f, \tS) := \oo{|\tS|}\sum_{(x,y)\in \tS} \pred{y \neq f(x)}.$
The optimal Bayes risk of $\bmu$ is $R^*_\bmu := \inf \err_{\bmu}(f),$ where the infimum is taken over all measurable classifiers $f:\X \rightarrow \Y$.
We say that $\bmu$ is \emph{realizable} when $R^*_\bmu = 0$.
We omit the overline in $\bmu$ in the sequel when there is no ambiguity.

For a finite labeled set $S\subseteq \X\times\Y$ and any $x\in\X$, let $X_{\nn}(x,S)$ be the nearest neighbor of $x$ with respect to $S$ and let $Y_{\nn}(x,S)$ be the nearest neighbor label of $x$ with respect to $S$:
\beq
(X_{\nn}(x,S),Y_{\nn}(x,S)) := \argmin_{(x',y')\in S} \dist(x, x'),
\eeq
where ties are broken arbitrarily.
The 1-NN classifier induced by $S$ is denoted by $h_{S}( x ) := Y_{\nn}(x,S)$.
The set of points in $S$, denoted by $\bm X = \{X_1,\ldots,X_{|S|}\} \subseteq \X$, induces a {\em Voronoi partition} of $\X$, $\Vor(\bm X) := \{V_1(\bm X),\dots, V_{|S|}(\bm X)\}$, where each Voronoi cell is 
\mbox{$V_i(\bm X) := \{ x \in \X : \argmin_{j \in \{1,\dots,|S|\}} \rho(x,X_j) = i \}$.}
By definition, $\forall x\in V_i(\bm X)$, $h_S(x) = Y_i$.

A 1-NN algorithm is a mapping from an i.i.d.~labeled sample $S_n \sim \bmu^n$ to a labeled set $\tS_n \subseteq \X \times \Y$, yielding the 1-NN classifier $h_{\tS_n}$. 
While the classic 1-NN algorithm sets $\tS_n \gets S_n$, in this work we study a compression-based algorithm which sets $\tS_n$ adaptively, as discussed further below.

A 1-NN algorithm is {\em strongly} Bayes-consistent on $\bmu$ if $\err(h_{\tS_n})$ converges to $R^*$ almost surely, that is $\P[\lim_{n \rightarrow \infty} \err(h_{\tS_n}) = R^*] = 1$. 
An algorithm is {\em weakly} Bayes-consistent on $\bmu$ if $\err(h_{\tS_n})$ converges to $R^*$ in expectation, $\lim_{n\to\infty}\E[\err(h_{\tS_n})]=R^*$.
Obviously, the former implies the latter. We say that an algorithm is Bayes-consistent on a metric space if it is Bayes-consistent on all distributions in the metric space. 

A convenient property that is used when studying the Bayes-consistency of algorithms in metric spaces is the \emph{doubling dimension}. 
Denote the open ball of radius $r$ around $x$ by $B_r(x) := \{x'\in\X : \dist(x, x')<r \}$ and let $\cB_r(x)$ denote the corresponding closed ball.
The doubling dimension of a metric space $(\X,\rho)$ is defined as follows. Let $n$ be the smallest number such that every ball in $\X$ can be covered by $n$ balls of half its radius, where all balls are centered at points of $\X$.
Formally, 
\[
n := 
\min\{n\in\nats:
\forall x\in\X,r>0, \;\; \exists x_1,\ldots,x_n\in\X \text{ s.t. } B_{r}(x) \subseteq \cup_{i=1}^n B_{r/2}(x_i)
\}.
\]
Then the doubling dimension of $(\X,\rho)$ is defined by $\ddim(\X,\rho) :=\log_2n$. 

For an integer $n$, let $[n] := \{1,\ldots,n\}$.
Denote the set of all index vectors of length $d$ by $I_{n,d} := [n]^d.$
Given a labeled set $S_n = (X_i,Y_i)_{i \in [n]}$ and any $\bm i = \{i_1,\ldots,i_d\} \in I_{n,d}$, denote the sub-sample of $S_n$ indexed by $\bm i$ by $S_n(\bm i) := \{(X_{i_1}, Y_{i_1}), \dots,(X_{i_d}, Y_{i_d})\}$.
Similarly, for a vector $\tbY = \{\tY_1,\ldots,\tY_d\} \in \Y^d$, denote by $S_n{(\bm i, \tbY)} := \{(X_{i_1}, \tY_{1}), \dots,(X_{i_d}, \tY_{d})\}$, namely the sub-sample of $S_n$ as determined by $\bm i$ where the labels are replaced with $\tbY$.
Lastly, for $\bm i,\bm j \in I_{n,d}$, we denote $S_n(\bm i; \bm j) := \{(X_{i_1}, Y_{j_1}), \dots,(X_{i_d}, Y_{j_d})\}.$

\section{1-NN majority-based compression}
\label{sec:compression_scheme}
In this work we consider the 1-NN majority-based compression algorithm proposed in \cite{kontorovichsabatourner16}, which we refer to as \algname.
This algorithm is based on constructing \emph{$\g$-nets} at different scales; for $\g > 0$ and $A \subseteq \X$, a set $\gnet \subseteq A$ is said to be a $\g$-net of $A$ if $\forall a\in A,\exists x\in\gnet:~ \dist(a,x)\leq\g$ and for all $x\neq x'\in\gnet$, $\dist(x,x')>\g$.%
\footnote{
   For technical reasons, having to do with the construction in \secref{infinite}, we depart slightly from the standard definition of a $\g$-net $\gnet \subseteq A$.
The classic definition requires that (i) $\forall a\in A,\exists x\in\gnet:~ \dist(a,x)<\g$ and (ii) $\forall x\neq x'\in \gnet:~ \dist(x,x')\ge\g$.
In our definition, the relations $<$ and $\ge$ in (i) and (ii) are
replaced by
$\le$ and $>$.
}
 
The algorithm (see \algref{simple}) operates as follows. 
Given an input sample $S_n$, whose set of points is denoted $\bm X_n = \{X_1,\ldots,X_n\}$, \algname\ considers all possible scales $\g > 0$.
For each such scale it constructs a $\g$-net of $\bm X_n$. Denote this $\g$-net by $\bm X(\g) := \{X_{i_1},\ldots,X_{i_{\k}}\}$, where $\k \equiv \k(\g)$ denotes its size and $\bm i \equiv \bm i(\g) := \{i_1,\ldots,i_{\k}\} \in I_{n,\k}$ denotes the indices selected from $S_n$ for this $\g$-net.
For every such $\g$-net, the algorithm attaches the labels $\tbY \equiv \tbY(\g)\in \Y^{\k}$, which are the empirical majority-vote labels in the respective Voronoi cells in the partition $\Vor(\bm X(\g)) = \{V_1,\ldots,V_{\k}\}$.
Formally, for $i \in [\k]$, 
\beqn\label{eq:maj}
\tY_i \in \argmax_{y \in \Y} |\{ j \in [n] \mid X_j \in V_{i}, Y_j = y\}|,
\eeqn
where ties are broken arbitrarily.
This procedure creates a labeled set $\tS_n(\g) := S_n(\bm i (\g), \tbY(\g))$ for every relevant $\g \in \{ \rho(X_i,X_j) \mid i,j \in [n]\} \setminus \{0\}.$
The algorithm then selects a single $\g$, denoted $\g^*\equiv \g^*_n$, and outputs $h_{\tS_n(\g^*)}$. The scale $\g^*$ is selected so as to minimize a generalization error bound, which upper bounds $\err(\tS_n(\g))$ with high probability. This error bound, denoted $Q$ in the algorithm, can be derived using a compression-based analysis, as described below. 

\begin{algorithm}[h]
\caption{\algname: 1-NN compression-based algorithm}
\label{alg:simple} 
\begin{algorithmic}[1]  
\Require Sample $S_n = (X_i, Y_i)_{i\in[n]}$, confidence $\delta$
\Ensure A 1-NN classifier
\State Let $\Gamma\gets\set{\rho(X_i,X_j) \mid i,j \in [n]} \setminus \{0\}$
\For{$\g \in \Gamma$}
\State Let $\bm X(\g)$ be a $\g$-net of $\{X_1,\ldots,X_n\}$
\State Let $\k(\g) \gets |\bm X(\g)|$
\State For each $i \in [\k(\g)]$, let $\tY_i$ be the majority label in $V_i(\bm X(\g))$ as defined in \eqref{maj}
\State Set $\tS_n(\g) \gets (\bm X(\g), \tbY(\g))$
\EndFor
\State Set $\a(\g) \gets \serr(h_{\tS_n(\g)}, S_n)$
\State Find $\g^*_n \in  \argmin_{\g\in\Gamma} Q(n,\a(\g), 2\k(\g), \delta)$, where $Q$ is, e.g.,  as in \eqref{KSUbound}
\State Set $\tS_n \gets \tS_n({\g^*_n})$
\State \Return $h_{\tS_n}$   
\end{algorithmic}
\end{algorithm}

We say that a mapping $S_n \mapsto \tS_n$ is a \emph{compression scheme} if there is a function $\C:\cup_{m = 0}^\infty (\X \times \Y)^m \rightarrow 2^{\X \times \Y}$, from sub-samples to subsets of $\X \times \Y$, such that for every $S_n$ there exists an $m$ and a sequence $\bm i \in I_{n,m}$ such that $\tS_n = \C(S_n(\bm i))$.
Given a compression scheme $S_n \mapsto \tS_n$ and a matching function $\C$, we say that a specific $\tS_n$ is an \emph{$(\alpha,m)$-compression} of a given $S_n$ if $\tS_n = \C(S_n(\bm i))$ for some $\bm i \in I_{n,m}$ and $\serr(h_{\tS_n},S_n) \leq \alpha$.
The generalization power of compression was recognized by \cite{floyd1995sample} and \cite{graepel2005pac}.
Specifically, it was
shown in \cite[Theorem 8]{gkn-jmlr17+aistats} that if the mapping $S_n \mapsto \tS_n$ is a compression scheme, then with probability at least $1-\delta$, for any $\tS_n$ which is an $(\alpha,m)$-compression of $S_n \sim \bmu^n$, we have
(omitting the constants, explicitly provided therein,
which do not affect our analysis)
\begin{equation}\label{eq:KSUbound}
\err(h_{\tS_n}) \leq  \frac{n}{n-m}\alpha + O(\frac{m\log(n) + \log(1/\delta)}{n-m}) + O(\sqrt{\frac{\frac{nm}{n-m}\alpha \log(n) + \log(1/\delta)}{n-m}}).
\end{equation}
Defining $Q(n,\a,m,\delta)$ as the RHS of \eqref{KSUbound}
provides \algname\ with a compression bound.
The following proposition shows that \algname\ is a compression scheme, which enables us to use \eqref{KSUbound} with the appropriate substitution.%
\footnote{ In \cite{kontorovichsabatourner16} the analysis was based on compression with side information, and does not extend to infinite $\Y$.}

\begin{proposition}
\label{prop:compression-scheme}
The mapping $S_n \mapsto \tS_n$ defined by \algref{simple} is a compression scheme whose output $\tS_n$ is a $(\serr(h_{\tS_n}), 2|\tS_n|)$-compression of $S_n$.
\end{proposition}

\begin{proof}
Define the function $\C$ by $\C((\bar{X}_i,\bar{Y}_i)_{i \in [2m]}) = (\bar{X}_i,\bar{Y}_{i + m})_{i \in [m]}$, and observe that for all $S_n$, we have $\tS_n = \C(S_n(\bm i(\g) ; \bm j(\g)))$, where $\bm i(\g)$ is the $\g$-net index set as defined above, and $\bm j(\g) = \{j_1,\ldots,j_{\k(\g)}\} \in I_{n, \k(\g)}$ is some index vector such that $\tY_i = Y_{j_i}$ for every $i \in [\k(\g)]$.
Since $\tY_i$ is an empirical majority vote, clearly such a $\bm j$ exists.
Under this scheme, the output $\tS_n$ of this algorithm is a $(\serr(h_{\tS_n}), 2|\tS_n|)$-compression.
\end{proof}

\algname\ is efficient, for any countable $\Y$.
Indeed, \algref{simple} has a naive runtime complexity of $O(n^4)$, since $O(n^2)$ values of $\g$ are considered and a $\g$-net is constructed for each one in time $O(n^2)$ (see \cite[Algorithm 1]{gknips14}). Improved runtimes can be obtained, e.g., using the methods in   \cite{KL04,DBLP:journals/tit/GottliebKK14+colt}.
In this work we focus on the Bayes-consistency of \algname, rather than optimize its computational complexity.
Our Bayes-consistency results below hold for \algname, whenever the generalization bound $Q(n,\alpha,m,\delta_n)$ satisfies the following properties:
\begin{enumerate}[label=\textbf{Property \arabic*},leftmargin=*]
\item  For any integer $n$ and $\delta \in (0,1)$, with probability $1 - \delta$ over the i.i.d.\! random sample $S_n \sim \bmu^n$, for all $\a \in [0,1]$ and $\k  \in [n]$:
If $\tS_n$ is an $(\a, \k)$-compression of $S_n$, then $\err(h_{\tS_n}) \leq Q(n,\alpha, m,\delta).$
\item $Q$ is monotonically increasing in $\a$ and in $\k$.
\item There is a sequence $\{\delta_n\}_{n = 1}^\infty$, $\delta_n \in (0,1)$ such that $\sum_{n=1}^\infty \delta_n < \infty$ and for all $m$, 
\[
\lim_{n\rightarrow \infty} \sup_{\a\in [0,1]} (Q(n,\a,m,\delta_n) - \a) = 0.
\]
\end{enumerate}
The compression bound in \eqref{KSUbound} clearly satisfies these properties. Note that Property 3 is satisfied by \eqref{KSUbound} using any convergent series $\sum_{n=1}^\infty \delta_n < \infty$ such that $\delta_n= e^{-o(n)}$;
in particular, the decay of $\delta_n$ cannot be too rapid.

\section{Main results}
\label{sec:main_results}
In this section we describe our main results. The proofs appear in subsequent sections.
First, we show that \algname\ is Bayes-consistent if the instance space has a finite doubling dimension. This contrasts with classical 1-NN, which is only Bayes-consistent if the distribution is realizable. 

\begin{theorem}
\label{thm:comp-consist}
Let $(\X,\rho)$ be a metric space with a finite doubling-dimension.
Let $Q$ be a generalization bound that satisfies Properties $1\mydash3$, and let $\delta_n$ be as stipulated by Property $3$ for $Q$. If the input confidence $\delta$ for input size $n$ is set to $\delta_n$, then the 1-NN classifier $h_{\tS_n({\g^*_n})}$ calculated by \algname\ is strongly Bayes consistent on $(\X,\rho)$:
\(
\P(\lim_{n \rightarrow \infty} \err(h_{\tS_n}) = R^*) = 1. 
\)
\end{theorem}

The proof, provided in \secref{consistency_proof}, closely follows the line
of reasoning in \cite{kontorovich2014bayes}, where the strong Bayes-consistency of an adaptive margin-regularized $1$-NN algorithm was proved, but with several crucial differences.
In particular, the generalization bounds used by \algname\ are purely compression-based, as opposed to the Rademacher-based generalization bounds used in \cite{kontorovich2014bayes}.
The former can be much tighter in practice and guarantee Bayes-consistency of \algname\ even for countably many labels.
This however requires novel technical arguments, which are discussed in detail in \appref{compare_to_aistats}.
Moreover, since the compression-based bounds do not explicitly depend on $\ddim$, they can be used even when $\ddim$ is infinite, as we do in \thmref{ksubayespreiss} below.
To underscore the subtle nature of Bayes-consistency, we note that the proof technique given here does not carry to an earlier algorithm, suggested in \cite[Theorem 4]{gknips14}, which also uses $\g$-nets.
It is an open question whether the latter is Bayes-consistent.

Next, we study Bayes-consistency of
\algname\ in infinite dimensions (i.e., with $\ddim = \infty$) --- in particular, in a setting where $k$-NN was shown by \cite{cerou2006nearest} not to be Bayes-consistent.
Indeed, a straightforward application of \cite[Lemma A.1]{cerou2006nearest}
yields the following result.

\begin{theorem}[C{\'e}rou and Guyader \cite{cerou2006nearest}]
\label{thm:knn-preiss}
There exists an infinite dimensional separable metric space $(\X,\dist)$ and a realizable distribution $\bmu$ over $\X\times\{0,1\}$ such that no $k_n$-NN learner satisfying $k_n/n\to0$ when $n \to \infty$ is Bayes-consistent under $\bmu$.
In particular, this holds for
any space and realizable distribution $\bmu$ that satisfy the following condition:
The set $C$  
of points labeled $1$ by 
$\bmu$ satisfies
\beqn\label{eq:preissprop}
\mu(C) > 0
\qquad\text{and}\qquad
\forall x \in C, \quad 
\lim_{r \rightarrow
  0} \frac{\mu(C \cap \cB_r(x))}{\mu(\cB_r(x))} = 0.
\eeqn
\end{theorem}

Since $\mu(C) > 0$, \eqref{preissprop} constitutes a violation of the Besicovitch covering property.
In doubling spaces, the Besicovitch covering theorem precludes such a violation \cite{MR0257325}.
In contrast, as \cite{preiss1979invalid,MR609946} show, in infinite-dimensional spaces this violation can in fact occur. Moreover, this is not an isolated pathology, as this property is shared 
by Gaussian Hilbert spaces \cite{MR1974687}.

At first sight, \eqref{preissprop} might appear to thwart any 1-NN algorithm applied to such a distribution. However, the following result shows that this is not the case: \algname\ is Bayes-consistent on a distribution with this property. 

\begin{theorem}\label{thm:ksubayespreiss}
There is a metric space equipped with a realizable distribution 
for which \algname\ is weakly Bayes-consistent, while any $k$-NN classifier necessarily is not.
\end{theorem}

The proof relies on a classic construction of Preiss \cite{preiss1979invalid} which satisfies \eqref{preissprop}.
We show that the structure of the construction, combined with the packing and covering properties of $\g$-nets, imply that the majority-vote classifier induced by {\em any} $\g$-net with a sufficienlty small $\g$ approaches the Bayes error.
%

We conclude the main results by posing intriguing open problems. 
\paragraph{Open problem 1.}
Does there exist a metric probability space on which some $k$-NN algorithm is consistent while \algname\ is not? Does there exist {\em any} separable metric space on which \algname\ fails?
\paragraph{Open problem 2.}
C{\'e}rou and Guyader \cite{cerou2006nearest} distill a certain Besicovitch condition 
which is necessary and sufficient for $k$-NN to be Bayes-consistent in a metric space. 
Our \thmref{ksubayespreiss} shows that the Besicovitch condition is {\em not} necessary for \algname\ to be Bayes-consistent. Is it sufficient?
What is a necessary condition?

\section{Bayes-consistency of \algname\ in finite dimensions}
\label{sec:consistency_proof}
In this section we give a high-level proof of \thmref{comp-consist}, showing that \algname\ is strongly Bayes-consistent in finite-dimensional metric spaces.
A fully detailed proof is given in \appref{finite_dim_proofs}.

Recall the optimal empirical error $\a_n^*\equiv \a(\g^*_n)$ and the optimal compression size $\k_n^*\equiv\k(\g^*_n)$ as computed by \algname.
As shown in \propref{compression-scheme}, the sub-sample $\tS_n(\g_n^*)$ is an $(\a_n^*, 2\k_n^*)$-compression of $S_n$.
Abbreviate the compression-based generalization bound used in \algname\ by
\[
Q_n(\alpha,m) := Q(n,\alpha,2m,\delta_n).
\]
To show Bayes-consistency, we start by a standard decomposition of the excess error over the optimal Bayes into two terms:
\beq
\err(h_{\tS_n(\g^*_n)}) - R^*  
=
\big(\err(h_{\tS_n(\g^*_n)}) - Q_n(\a_n^*,\k_n^*) \big)
+
\big(Q_n(\a_n^*,\k_n^*) - R^*\big)
=:
T_I(n) + T_{II}(n),
\eeq
and show that each term decays to zero with probability one.
For the first term, Property 1 for $Q$, together with the Borel-Cantelli lemma, readily imply $\limsup_{n\to\infty} T_I(n) \leq 0$ with probability one. 
The main challenge is showing that $\limsup_{n\to\infty} T_{II}(n) \leq 0$ with probability one.
We do so in several stages:
\begin{enumerate}
\item Loosely speaking, we first show (\lemref{richness}) that the Bayes error $R^*$ can be well approximated using 1-NN classifiers defined by the {\em true} (as opposed to empirical) majority-vote labels over fine partitions of $\X$.
In particular, this holds for any partition induced by a $\g$-net of $\X$ with a sufficiently small $\g>0$.
This approximation guarantee relies on the fact that in finite-dimensional spaces, the class of continuous functions with compact support is dense in $L_1(\mu)$ (\lemref{dense_cont}).

\item Fix $\tilde \g>0$ sufficiently small such that any true majority-vote classifier induced by a $\tilde\g$-net has a true error close to $R^*$, as guaranteed by stage 1. 
Since for bounded subsets of finite-dimensional spaces the size of any $\g$-net is finite, the empirical error of any majority-vote $\g$-net almost surely converges to its true majority-vote error as the sample size $n\to\infty$.
Let $n(\tilde \g)$ sufficiently large such that $Q_{n(\tilde \g)}(\a(\tilde \g), \k(\tilde \g))$ as computed by \algname\ for a sample of size $n(\tilde \g)$ is a reliable estimate for the true error of $h_{\tS_{n(\tilde \g)}(\tilde\g)}$.
		
\item Let $\tilde\g$ and $n(\tilde\g)$ be as in stage 2.
Given a sample of size $n = n(\tilde\g)$,
recall that \algname\ selects an optimal $\g^*$ such that $Q_n(\a(\g), \k(\g))$ is minimized over all $\g>0$.
For margins $\g \ll \tilde\g$, which are prone to over-fitting, $Q_n(\a(\g), \k(\g))$ is not a reliable estimate for $h_{\tS_n(\g)}$ since compression may not yet taken place for samples of size $n$.
Nevertheless, these margins are discarded by \algname\ due to the penalty term in $Q$. 
On the other hand, for $\g$-nets with margin $\g \gg \tilde\g$, which are prone to under-fitting, the true error is well estimated by $Q_n(\a(\g), \k(\g))$.
It follows that \algname\ selects $\g_n^*\approx \tilde\g$ and  $Q_n(\a_n^*, \k_n^*)\approx R^*$, implying $\limsup_{n\to\infty} T_{II}(n) \leq 0$ with probability one.
\end{enumerate}
As one can see, the assumption that $\X$ is finite-dimensional plays a major role in the proof.
A simple argument shows that the family of continuous functions with compact support
is no longer dense in $L_1$ in infinite-dimensional spaces.
In addition, $\g$-nets of bounded subsets in infinite dimensional spaces need no longer be finite.

\section{On Bayes-consistency of NN algorithms in infinite dimensions}
 \label{sec:infinite}

In this section we study the Bayes-consistency properties of 
1-NN algorithms on a classic infinite-dimensional construction of Preiss \cite{preiss1979invalid}, which we describe below in detail.
This construction was first introduced as a concrete example showing that in infinite-dimensional spaces the Besicovich covering theorem \cite{MR0257325} can be strongly violated, as manifested in \eqref{preissprop}.

\begin{example}[Preiss's construction]\label{ex:preiss}
The construction (see Figure \ref{fig:preiss_construction}) defines an infinite-dimensional metric space $(\X,\rho)$ and a realizable measure $\bmu$ over $\X \times \Y$ with the binary label set $\Y = \{0,1\}$. 
It relies on two sequences: a sequence of natural numbers $\{N_k\}_{k \in \N}$ and a sequence of positive numbers $\{a_k\}_{k \in \N}$.
The two sequences should satisfy the following:
\begin{align}\label{eq:nkak}
\textstyle
\sum_{k=1}^\infty a_k N_1 \dots N_k = 1;
\quad \lim_{k\to\infty} a_{k } N_1 \dots  N_{k+1} = \infty;
\quad \text{and} \quad
 \lim_{k\to\infty} N_k = \infty.
\end{align}
These properties are satisfied, for instance, by setting $N_k := k!$ and $a_k := 2^{-k}/\prod_{i\in[k]}N_i$.
Let $\fins$ be the set of all {finite} sequences $(z_1,\dots,z_k)_{k\in\N}$ of natural numbers such that $z_i \leq N_i$, and let $\infs$ be the set of all {infinite} sequences $(z_1,z_2,\dots)$ of natural numbers such that $z_i \leq N_i$.

Define the example space $\X := \fins \cup \infs$
and denote $\seq{k} := 2^{-k}$, where $\seq{\infty} := 0$.
The metric $\rho$ over $\X$ is defined as follows: for $x,y \in \X$, denote by $\ca{x}{y}$ their longest common prefix.
Then,
\[
\rho(x,y) = 
(\seq{|\ca{x}{y}|}  - \seq{|x|}) + (\seq{|\ca{x}{y}|}  - \seq{|y|}).
\]
It can be shown (see \cite{preiss1979invalid}) that $\rho(x,y)$ is a metric; in fact, it embeds isometrically into the square norm metric of a Hilbert space.

To define $\mu$, the marginal measure over $\X$, let $\nu_\infty$ be the uniform product distribution measure over $\infs$, that is: for all $i \in \N$, each $z_i$ in the sequence $z= (z_1,z_2,\dots)\in\infs$ is independently drawn from a uniform distribution over $[N_i]$.
Let $\nu_0$ be an atomic measure on $\fins$ such that for all $z \in \fins$,
$\nu_0(z) = a_{|z|}$.
Clearly, the first condition in \eqref{nkak} implies $\nu_0(\fins) = 1$.
Define the marginal probability measure $\mu$ over $\X$  by 
\[
\forall A \subseteq \fins \cup \infs, \quad \mu(A) := \alpha \nu_\infty(A) + (1-\alpha)\nu_0(A). 
\]
In words, an infinite sequence is drawn with probability $\alpha$ (and all such sequences are equally likely), or else a finite sequence is drawn (and all finite sequences of the same length are equally likely).
Define the realizable distribution $\bmu$ over $\X \times \Y$ by setting the marginal over $\X$ to $\mu$, and by setting the label of $z \in \infs$ to be $1$ with probability $1$ and the label of $z \in \fins$ to be $0$ with probability $1$. 
\end{example}

\begin{figure}%
\center
\begin{tikzpicture}[scale=1.5,font=\footnotesize]
\tikzset{
solid node/.style={circle,draw,inner sep=1.5,fill=black},
hollow node/.style={inner sep=1}
}
\tikzstyle{level 1}=[level distance=8mm,sibling distance=8mm]
\tikzstyle{level 2}=[level distance=8mm,sibling distance=6mm]
\tikzstyle{level 3}=[level distance=8mm,sibling distance=8mm]
\tikzstyle arrowstyle=[scale=1]
\tikzstyle directed=[postaction={decorate,decoration={markings,
mark=at position .5 with {\arrow[arrowstyle]{stealth}}}}]
\node(0)[solid node,label = right:{$z_{1:k-2}$}]{}
child[dashed]{node(1)[hollow node]{}
edge from parent node[left,xshift=-3]{}
}
child[dashed]{}
child[dashed]{}
child{node(2)[solid node,label=right:{$z_{1:k-1}$}]{}
child{node(3)[solid node]{}
child[dashed]{node[hollow node,label=below:{}]{}edge from parent node[left]{$\g_{k}$} }
child[dashed]{node[hollow node,label=below:{}]{} edge from parent node[right]{$\g_k$}}
edge from parent node[left]{$\g_k$}}
child[dashed]{}
child[dashed]{}
child{node(4)[solid node, label=right:{$z_{1:k}$}]{}
child[dashed]{node(5)[solid node,label=below:{$z$}]{}edge from parent node[left]{$\g_k$} }
child[dashed]{node(6)[hollow node,label=below:{}]{} edge from parent node[right]{$\g_k$}}
edge from parent node[right]{$\g_k$}
}
edge from parent node[right,xshift=3]{$\g_{k-1}$}
};
\draw[dashed,rounded corners=10]($(5) + (-.2,-.25)$) ($(4) +(.2,-.25)$);
\node at ($(3)!.5!(4)$) {};
\node[hollow node] (RB) at ($(6) + (0.5,0)$){};
\node[hollow node] (RRB) at ($(RB) + (1.2,0)$){$C = \infs$};
\node[hollow node] (LB) at ($(RB) - (4.5,0)$){};
\draw[thick] (LB) -- (RB);
\draw[thick, dashed] (LB) -- ($(LB)-(.5,0)$);
\draw[thick, dashed] (LB) -- ($(RB)+(.5,0)$);
\draw[dashed] (0) -- ($(0)+ (0,.3)$);
\draw[gray] plot  coordinates {
($(5)+(-.3,0)$)
($(5)+(-.2,-.3)$)
($(6)+(.15,-.3)$)
($(6)+(.3,0)$)
($(4)+(.3,0)$)
($(2)+(-.03,.26)$)
($(2)+(-.25,0)$)
($(4)+(-.25,0)$)
($(5)+(-.3,0)$)
};
\end{tikzpicture}
\caption{Preiss's construction. Encircled is the closed ball $\cB_{\g_{k-1}}(z)$ for some $z\in C$.} %
\label{fig:preiss_construction}%
\end{figure}
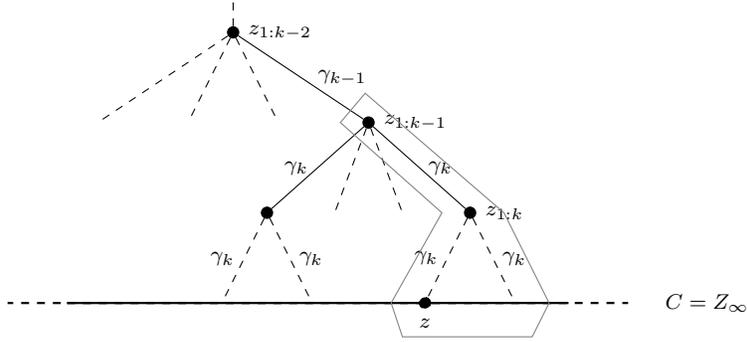

As shown in \cite{preiss1979invalid}, this construction satisfies \eqref{preissprop} with $C=\infs$ and $\mu(C)=\alpha>0$.
It follows from \thmref{knn-preiss} that no $k$-NN algorithm is Bayes-consistent on it.
In contrast, the following theorem shows that \algname\ is weakly Bayes-consistent on this distribution. \thmref{ksubayespreiss} immediately follows from the this result. 

\begin{theorem}\label{thm:ksupreissemp}
Assume $(\X,\rho),\Y$ and $\bmu$ as in \exref{preiss}.
\algname\ is weakly Bayes-consistent on $\bmu$. 
\end{theorem}

The proof, provided in \appref{preiss_construction}, first characterizes the Voronoi cells for which the true majority-vote yields a significant error for the cell (\lemref{preiss_gnet_aux}).
In finite-dimensional spaces, the total measure of all such ``bad'' cells can be made arbitrarily close to zero by taking $\g$ to be sufficiently small, as shown in \lemref{richness} of \thmref{comp-consist}.
However, it is not immediately clear whether this can be achieved for the infinite dimensional construction above.

Indeed, we expect such bad cells, due to the unintuitive property that for any $x\in C$, we have $\mu(\cB_\g(x) \cap C)/ {\mu(\cB_\g(x))} \to0$ when ${\g\to 0}$, and yet $\mu(C)> 0$.
Thus, if for example a significant portion of the set $C$ (whose label is 1) is covered by Voronoi cells of the form $V = \cB_{\g}(x)$ with $x\in C$, then for all sufficiently small $\g$, each one of these cells will have a true majority-vote $0$. Thus a significant portion of $C$ would be misclassified.
However, we show that by the structure of the construction, combined with the packing and covering properties of $\g$-nets, we have that in {\em any} $\g$-net, the total measure of all these ``bad'' cells goes to 0 when $\g\to0$, thus yielding a consistent classifier.

\begin{remark} In a previous version of the manuscript the following erroneous claim was made (Theorems 5 and 7 in previous version): "For the example constructed in Theorem \ref{thm:ksupreissemp}, there exists a sequence of [countable] Voronoi partitions with a vanishing diameter, such that the induced true majority-vote classifiers are not Bayes consistent". Unfortunately, the proof contains technical errors. In fact, such sequences do not exist, as recently established by \cite[Lemma 4.3]{HKSW2019}.
\end{remark}

%
%

\paragraph{Acknowledgments.}
We thank
Fr\'ed\'eric C\'erou
for the numerous fruitful discussions and helpful feedback on an earlier draft.
Aryeh Kontorovich was supported in part by
the Israel Science Foundation (grant No. 755/15),
Paypal and IBM.
Sivan Sabato was supported in part by the Israel Science Foundation (grant No. 555/15).

\appendix

\section{Background on metric measure spaces}
\label{ap:metric_space_basics}
Here we provide some general relevant background on metric measure spaces.
Our metric space $(\X,\rho)$ is
doubling,
but in this section finite diameter is not assumed.
We recall some standard definitions.
A topological space is {\em Hausdorff} if every two distinct points
have disjoint neighborhoods. It is a standard (and obvious) fact
that all metric spaces are Hausdorff.

A metric space $\X$ is {\em complete} if every Cauchy sequence
converges to a point in $\X$.
Every metric space may be completed
by (essentially) adjoining to it the limits of all of its Cauchy
sequences \cite[Exercise 3.24]{MR0385023};
moreover, the completion is unique up to isometry
\cite[Section 43, Exercise 10]{MR0464128}.
We
implicitly
assume
throughout the paper
that
$\X$
is
complete.
Closed subsets of complete metric spaces
are also complete metric spaces under the inherited metric.

A
topological
space $\X$ is {\em locally compact} if every point $x\in\X$
has a compact neighborhood. It is a standard and easy fact
that complete doubling spaces are locally compact.
Indeed, consider any $x\in\X$
and the open $r$-ball about $x$, $B_r(x):=\set{y\in\X:\rho(x,y)<r}$.
We must show that 
$\cl(B_r(x))$
--- the closure of $B_r(x)$
--- 
is compact. To this end, it suffices to show that
$\cl(B_r(x))$
is {\em totally bounded}
(that is, has a finite $\eps$-covering number for each $\eps>0$),
since in complete metric spaces, a set is compact iff it is closed and
totally bounded \cite[Theorem 45.1]{MR0464128}.
Total boundedness follows immediately from the doubling property.
The latter posits a
constant $k$
and some
$x_1,\ldots,x_k\in\X$ such that
$B_r(x)\subseteq\cup_{i=1}^k B_{r/2}(x_i)$.
Then certainly
$
\cl(B_r(x))
\subseteq
\cup_{i=1}^k B_{2r/3}(x_i).
$
We now apply the doubling property recursively to each of the
$B_{2r/3}(x_i)$, until the radius of the covering balls becomes smaller than $\eps$.

We now recall some standard facts from measure theory.
Any topology on $\X$ (and in particular, the one induced by the metric
$\rho$),
induces the Borel $\sigma$-algebra $\mathscr{B}$.
A Borel probability measure is a function $\mu:\mathscr{B}\to[0,1]$
that is countably additive and normalized by $\mu(\X)=1$.
The latter is {\em complete} if for all
$A\subseteq B\in\mathscr{B}$ for which $\mu(B)=0$, we also
have $\mu(A)=0$. Any Borel $\sigma$-algebra may be completed
by defining the measure of any subset of a measure-zero set to be zero
\cite[Theorem 1.36]{rudin87}.
We
implicitly
assume
throughout the paper
that
$(\X,\mathscr{A},\mu)$ is a complete measure space,
where $\mathscr{A}$ contains all of the Borel sets.

The measure $\mu$ is said to be {\em outer regular}
if it can be approximated from above by open sets:
For every $E\in\mathscr{A}$, we have
\beq
\mu(E)=\inf\set{\mu(V):E\subseteq V, V~\text{open}}.
\eeq
A corresponding {\em inner regularity}
corresponds to approximability from below by compact sets:
For every $E\in\mathscr{A}$,
\beq
\mu(E)=\sup\set{\mu(K):K\subseteq E, K~\text{compact}}.
\eeq
The measure $\mu$ is {\em regular} if it is both inner and outer regular.
Any probability measure defined on the Borel $\sigma$-algebra
of a metric space is regular \cite[Lemma 1.19]{Kallenberg02}.
(Dropping the ``metric'' or ``probability'' assumptions
opens the door to various
exotic pathologies
\cite[Chapter 7]{MR2267655}, \cite[Exercise 2.17]{rudin87}.)

Finally, we have the following technical result, adapted
from \cite[Theorem 3.14]{rudin87}
to our setting:
\begin{theorem}
\label{thm:cc-dense}
Let $\X$ be a complete doubling metric space equipped
with a complete
probability measure $\mu$, such that all Borel sets are $\mu$-measurable.
Then $C_c(\X)$ (the collection of continuous functions with compact support)
is dense in $L_1(\mu)$.
\end{theorem}

\section{Bayes-consistency proof of \algname\ in finite dimensions}
\label{ap:finite_dim_proofs}
In this section we prove \thmref{comp-consist} in full detail.
Let $(\X,\rho)$ be a metric space with doubling-dimension $\ddim<\infty$.
Given a sample $S_n\sim \bmu^n$, we abbreviate the optimal empirical error $\a_n^*=\a(\g^*_n)$ and the optimal compression size $\k_n^*=\k(\g^*_n)$ as computed by \algname.
As shown in \secref{compression_scheme}, the labeled set $\tS_n(\g_n^*)$ computed by {\algname} is an
$(\a_n^*, 2\k_n^*)$-compression of the sample $S_n$.
For brevity we denote
\[
Q_n(\alpha,m) := Q(n,\alpha,2m,\delta_n).
\]
To prove Theorem \ref{thm:comp-consist} we decompose the excess error over the Bayes into two terms:
\beq
\err(h_{\tS_n(\g^*_n)}) - R^*  
&= &
\big(\err(h_{\tS_n(\g^*_n)}) - Q_n(\a_n^*,\k_n^*) \big)
+
\big(Q_n(\a_n^*,\k_n^*) - R^*\big)
\\
&=:&
T_I(n) + T_{II}(n),
\eeq
and show that each term decays to zero with probability one.

For the first term, $T_I(n)$, from Property 1 of generalization bound $Q$, we have that for any $n> 0$,
\beqn
\label{eq:termI_bound}
\P_{S_n}\!
\left[
\err(h_{\tS_n(\g^*_n)}) - Q_n(\a_n^*,\k_n^*) > 0
\right] 
\leq \delta_n.
\eeqn
Since $\sum \delta_n <\infty$, the Borel-Cantelli lemma implies $\limsup_{n\to\infty} T_I(n) \leq 0$ with probability $1$. 

The main challenge is to prove that $\limsup_{n\to\infty} T_{II}(n) \leq 0$ with probability one.
We begin by showing that the Bayes error $R^*$ can be approached using classifiers defined by the true majority-vote labeling over fine partitions of $\X$.
Formally, let $\SP = \{\sp_1,\dots\}$ be a finite partition of $\X$, and
define the function
$I_\SP: \X \to \SP$ such that $I_\SP(x)$ is the unique $\sp\in\SP$
for which $x\in\sp$.
For any measurable set
$\emptyset\neq E\subseteq\X$
define the true majority-vote label $y^*(E)$ by
\beqn
\label{eq:bar y(E)}
y^*(E)
= \argmax_{y\in\Y} \P_\bmu(Y=y \gn X\in E),
\eeqn
where ties are broken lexicographically.
To ensure that $y^*$ is always well-defined, when $E=\emptyset$ we arbitrarily define it to be the lexicographically first $y\in\Y$.
Given $\SP$ and a measurable set $M \subseteq \X$, consider the true majority-vote classifier $h_{\SP,M}^*:\X\to\Y$ given by
\beqn
\label{eq:Strue2}
h_{\SP,M}^*(x) = y^*(I_\SP(x) \cap M).
\eeqn
Note that if $x \notin M$, this classifier attaches a label to $x$ based on the true majority-vote in a set that does not contain $x$.
To bound the error of $h_{\SP,M}^*$ for any conditional distribution of labels, we use the fact that on doubling metric spaces, continuous functions are dense in $L_1(\mu)$.
\begin{lemma}
\label{lem:dense_cont}
For every probability measure $\mu$ on a doubling metric space $\X$, the set of continuous functions $f:\X\to\R$ with compact support is dense in $L_1(\mu)= \{f: \int_{\X} {\abs{f}} d\mu(x)<\infty\}$. Namely, for any $\eps>0$ and $f\in L_1(\mu)$ there is a continuous function $g\in L_1(\mu)$ with compact support such that
\(
\int_{\X} \abs{f-g} d\mu(x) < \eps.
\)
\end{lemma}
\begin{proof}
  This is stated as \thmref{cc-dense}
  in \appref{metric_space_basics}.
\end{proof}

We have the following uniform approximation bound for the error of classifiers in the form of (\ref{eq:Strue2}), essentially extending the approximation analysis done in the proof of \cite[Theorem 21.2]{devroye2013probabilistic} for the special case $|\Y|=2$ and $\X=\R^d$ to the more general multi-class problem in doubling metric spaces.

\begin{lemma}
\label{lem:richness}
Let $\bmu$ be a probability measure on $\X\times\Y$ where $\X$ is a doubling metric space.
For any $\nu>0$, there exists a diameter $\beta=\beta(\nu)>0$ such that for any finite measurable partition $\SP = \{\sp_1,\dots\}$ of $\X$ and any measurable set $M \subseteq \X$ satisfying 
  \begin{itemize}
  \item[(i)] $\mu(\X\setminus M) \leq \nu$
  \item[(ii)] $\diam(\SP \cap M)\leq \beta$,
  \end{itemize}
  the true majority-vote classifier $h_{\SP,M}^*$ defined in (\ref{eq:Strue2}) satisfies
$\err(h_{\SP,M}^*) \leq R^* + 5\nu$.
\end{lemma}

\begin{proof} Let $\eta_y: \X \to [0,1]$ be the conditional probability function for label $y\in\Y$,
\beq
\eta_y(x) = \P_\bmu( Y = y \gn X = x).
\eeq
Define $\teta_y: \X \to [0,1]$ as $\eta_y$'s conditional expectation function with respect to $(\SP,M)$,
\beq
\teta_y(x) = \P_\bmu(Y = y \gn X \in I_\SP(x) \cap M)
= \frac{\int_{I_\SP(x) \cap M} \eta_y(z) d\mu(z) }{\mu(I_\SP(x) \cap M)}.
\eeq
(And, say, $\teta_y(x) =
\pred{y=1}
$ for $x\notin M$.)
Note that $(\teta_y)_{y\in\Y}$ are piecewise constant on the cells of the restricted partition $\SP \cap M$.
By definition, the Bayes classifier $h^*$ and the true majority-vote classifier $h^*_{\SP,M}$ satisfy
\beq
h^*(x) &=& \argmax_{y\in\Y} \eta_y(x),
\\
h^*_{\SP,M}(x) &=& \argmax_{y\in\Y} \teta_y(x).
\eeq
It follows that
\beq
\P_\bmu( h^*_{\SP,M}(X) \neq Y  \gn X = x) - \P_\bmu( h^*(X) \neq Y  \gn X = x)
&=& \eta_{h^*(x)}(x) - \eta_{h^*_{\SP,M}(x)}(x) 
\\ &=&  \max_{y\in\Y} \eta_y(x) - \max_{y\in\Y} \teta_y(x)
\\
&\leq& \max_{y\in\Y} |\eta_y(x) - \teta_y(x) |.
\eeq
By condition \textit{(i)} in the theorem statement, $\mu(\X\setminus M)\leq \nu $.
Thus,
\beqn
\nonumber
\err(h^*_{\SP,M}) - R^* & = &
\P_\bmu( h^*_{\SP,M}(X) \neq Y) - \P_\bmu( h^*(X) \neq Y)
\\
\nonumber
& \leq & \mu(\X\setminus M) +  \int_{ M} \max_{y\in\Y}|\eta_y(x) - \teta_y(x)| d\mu(x)
\\
\nonumber
& \leq & \nu +  \sum_{y\in\Y}\int_{ M} |\eta_y(x) - \teta_y(x)| d\mu(x). 
\eeqn
Let $\Y' \subseteq \Y$ be a finite set of labels such that $\P_\bmu[Y \in \Y'] \geq 1-\nu$. Then 
\beqn
\label{eq:err_R_eta_teta}
\err(h^*_{\SP,M}) - R^* \leq 2\nu +  \sum_{y\in\Y'}\int_{ M} |\eta_y(x) - \teta_y(x)| d\mu(x).
\eeqn
To bound the integrals in (\ref{eq:err_R_eta_teta}), we approximate $(\eta_y)_{y\in\Y}$ with functions from the dense set of continuous functions, by applying \lemref{dense_cont}.
Since $\eta_y \in L_1(\mu)$ for all $y\in\Y'$ and $|\Y'|<\infty$,
Lemma \ref{lem:dense_cont} implies that there are $|\Y'|$ continuous functions $(r_y)_{y\in\Y'}$ with compact support such that
\beqn
\label{eq:eta_r_eps}
\max_{y\in\Y'} \int_{\X} |\eta_y(x) - r_y(x)| d\mu(x)  \leq \nu/ |\Y'|.
\eeqn
Similarly to $(\teta_y)_{y\in\Y'}$, define the piecewise constant functions $(\tr_y)_{y\in\Y'}$ by
\beq
\tr_y(x) = \E_\mu[r_y(X) \gn X \in I_\SP(x) \cap M ]
= \frac{\int_{I_\SP(x) \cap M} r_y(z) d\mu(z) }{\mu(I_\SP(x) \cap M)}.
\eeq
We bound each integrand in (\ref{eq:err_R_eta_teta}) by
\beqn
\label{eq:abs_chain}
|\eta_y(x) - \teta_y(x)| \leq |\eta_y(x) - r_y(x)| + |r_y(x) - \tr_y(x)| + | \tr_y(x) - \teta_y(x)|.
\eeqn
The integral of the first term in (\ref{eq:abs_chain}) is smaller than $\nu/|\Y'|$ by the definition of $r_y$ in (\ref{eq:eta_r_eps}).
For the integral of the third term in (\ref{eq:abs_chain}),
\beq
\int_{M} \abs{ \tr_y(x) - \teta_y(x)} d\mu(x)
&=& \sum_{\sp\in\SP} 
\abs{\E_\mu[r_y(X) \pred{X \in \sp \cap M}] - \E_\mu[\eta_y(X)\pred{X \in \sp \cap M}]}
\\
&=&
\sum_{\sp\in\SP} \abs{\int_{\sp \cap M} r_y(x)d\mu(x) - \int_{\sp \cap M} \eta_y(x)d\mu(x)}
\\
&=&
\sum_{\sp\in\SP} \abs{\int_{\sp \cap M} (r_y(x)- \eta_y(x)) d\mu(x)}
\\
&\leq & \int_{M} \abs{r_y(x) - \eta_y(x)} d\mu(x) \;\leq\; \nu/|\Y'|.
\eeq
Finally, for the integral of the second term in (\ref{eq:abs_chain}),
\beq
&& \int_{M} |r_y(x) - \tr_y(x)|d\mu(x)
\\
&& \quad = 
\sum_{\sp\in\SP: \mu( \sp \cap M) \neq 0} \int_{\sp\cap M} \abs{r_y(x) - \frac{\E_\mu[r_y(X) \pred{X \in \sp \cap M}]}{\mu(\sp \cap M)}} d\mu(x)
\\
&& \quad = 
\sum_{\sp\in\SP: \mu( \sp \cap M) \neq 0}\frac{1}{\mu(\sp \cap M)} 
\int_{\sp\cap M} \abs{r_y(x) \mu(\sp \cap M) - \E_\mu[r_y(X) \pred{X \in \sp \cap M}] }
d\mu(x)
\\
&& \quad = 
\sum_{\sp\in\SP: \mu( \sp \cap M) \neq 0}\frac{1}{\mu(\sp \cap M)} 
\int_{\sp\cap M} \abs{r_y(x) \int_{\sp \cap M}d\mu(z) -
\int_{\sp \cap M} r_y(z) d\mu(z)}
d\mu(x)\\
&& \quad = 
\sum_{\sp\in\SP: \mu( \sp \cap M) \neq 0}\frac{1}{\mu(\sp \cap M)} 
\int_{\sp\cap M} \abs{ \int_{\sp \cap M}(r_y(x) - r_y(z))\, d\mu(z)}
d\mu(x)
\\
&& \quad \leq 
\sum_{\sp\in\SP: \mu( \sp \cap M) \neq 0}
\frac{1}{\mu(\sp \cap M)} 
\int_{\sp\cap M} \int_{\sp\cap M} \abs{r_y(x) -  r_y(z)} d\mu(x) d\mu(z).
\eeq
Since
$|\Y'|<\infty$
and any continuous function with compact support
is uniformly continuous on all of $\X$,
the collection $\set{r_y:y\in\Y'}$ is
equicontinuous.
Namely,
there exists a diameter $\beta = \beta(\nu)> 0$ such that for any
$A\subseteq \X$ with $\diam(A) \leq \beta$,
$\max_{y\in\Y'}\abs{r_y(x) - r_y(z)} \leq \nu/|\Y'|$
for every $x,z\in A$  (note that $\beta(\nu)$ does not depend on $(\SP,M)$).
By condition \textit{(ii)} in the theorem statement, $\diam(\sp \cap M) \leq \beta$ for all $\sp\in\SP$.
Hence,
\beq
\frac{1}{\mu(\sp \cap M)} \int_{\sp\cap M} \int_{\sp\cap M} \abs{r_y(x) -  r_y(z)} d\mu(x) d\mu(z) \leq \frac{\nu}{|\Y'|} \mu(\sp \cap M).
\eeq
Summing over all cells $\sp\in\SP$ with $\mu( \sp \cap M) \neq 0$, the integral of the second term in (\ref{eq:abs_chain}) satisfies
\beq
\int_{M} |r_y(x) - \tr_y(x)|d\mu(x) \leq \nu/|\Y'|.
\eeq
Putting the bounds for the three terms together,
\beq
\sum_{y\in\Y'}\int_{ M} |\eta_y(x) - \teta_y(x)| d\mu(x)
\leq 
\sum_{y\in\Y'} \frac{3\nu}{|\Y'|} = 3\nu.
\eeq
Applying this bound to (\ref{eq:err_R_eta_teta}), we conclude
\(
\err(h^*_{\SP,M}) - R^* \leq 5\nu.
\)
\end{proof}

Next, we prepare to use \lemref{richness} to show that the generalization bound $Q_n(\a_n^*, \k_n^*)$ also approaches the Bayes error $R^*$, thus proving $\limsup_{n\to\infty} T_{II}(n) \leq 0$ with probability one. Given $\eps>0$, fix 
\beqn 
\label{eq:gammaeps}
\g=\g(\eps) = \beta(\epsilon/10)/4,
\eeqn
where $\beta$ is as guaranteed by \lemref{richness}.
Let $\tS_n(\g) = (\bm X(\g),\tbY)$ be the labeled set calculated by \algref{simple} for this value of $\g$. Let $\a(\g)$ and $\k(\g)$ as defined in \algref{simple}.
We show that there exist $N = N(\eps) > 0$, a universal constant $c>0$, and a function $C(\g,\eps)>0$ that does not depend on $n$, such that $\forall n \geq N$,
\beqn
\label{eq:termII_bound_fixed_g}
\P_{S_n}[
Q_n(\a(\g), \k(\g))>R^* + \eps]
\leq
 C(\g,\eps)e^{-c n\eps^2}.
\eeqn

Note that since \algref{simple} finds $\g_n^*$ such that
\beq
Q_n(\a_n^*, \k_n^*) = \min_{\g'} Q_n(\a(\g'),\k(\g')) \leq Q_n(\a(\g), \k(\g)),
\eeq
the bound in (\ref{eq:termII_bound_fixed_g}) readily implies that $\forall n \geq N$,
\beqn
\label{eq:termII_bound}
\P_{S_n}[Q_n(\a_n^*, \k_n^*) > R^* + \eps] 
\leq  C(\g,\eps)e^{-cn\eps^2}.
\eeqn
By the Borel-Cantelli lemma, this implies that with probability one, 
\beq
\limsup_{n\rightarrow \infty} T_{II}(n) = \limsup_{n\rightarrow \infty} (Q_n(\a_n^*, \k_n^*) - R^*) \leq 0.
\eeq 
Since $\forall n, T_I(n) + T_{II}(n) \geq 0$, this implies $\lim_{n\to\infty} T_{II}(n) = 0$ with probability one, thus completing the proof of \thmref{comp-consist}.

We now proceed to prove (\ref{eq:termII_bound_fixed_g}). 
For $A \subseteq \X$, denote $\UB_\g(A) := \cup_{x\in A} B_\g(x)$ and consider the random variable
\beq
\missmass_\g(S_n) := \mu( \X \setminus \textup{\UB}_\g(S_n)),
\eeq
also known as the $\g$-missing mass of $S_n$.
We bound
\beqn
\nonumber
&& \P_{S_n}[ Q_n(\a(\g),\k(\g)) > R^* + \eps ]
\\
\label{eq:split_miss}
&&\quad \leq
\P_{S_n}[ Q_n(\a(\g),\k(\g)) > R^* + \eps \;\wedge\; \missmass_\g(S_n) \leq \frac{\eps}{10}]
+ \P_{S_n}[ \missmass_\g(S_n) > \frac{\eps}{10} ].
\eeqn
To bound the second term in (\ref{eq:split_miss}), we apply the following theorem, due to \cite{berend2012missing,berend2013concentration}, bounding the mean and deviation of the $\g$-missing mass $\missmass_\g(S_n)$.
Denote
\beq
\Ng := \ceil{\tfrac{\diam(\X)}{\g}}^{\ddim} < \infty.
\eeq

\begin{theorem}[\cite{berend2012missing,berend2013concentration}]
\label{thm:missing_mass}
Let $(\X,\dist)$ be a doubling metric space and let 
$S_n \sim\mu^n$.
Then
\beqn
\label{eq:miss_mass_mean}
\E_{S_n}[\missmass_\g(S_n)]
\leq
\frac{\mathcal{N}_{\g/2}}{e n}
\eeqn
and for any $\xi>0$,
\beqn
\label{eq:miss_mass_conc}
\P_{S_n}\Big(
\missmass_\g(S_n) > \E_{S_n}[\missmass_\g(S_n)] + \xi
\Big)
\leq 
\exp\left(- n \xi^2\right).
\eeqn
\end{theorem}
Taking $n$ sufficiently large so that $\frac{\mathcal{N}_{\g/2}}{e n} \leq \eps/20$ and applying Theorem \ref{thm:missing_mass} with $\xi=\eps/20$, we have 
\beqn
\label{eq:Psi_n}
\P_{S_n}[\missmass_\g(S_n) > \eps/10] \leq e^{-\frac{n \eps^2}{400}}.
\eeqn

We are left to bound the first term in \eqref{split_miss}.
To this end, we use the fact that any $\g$-net $\gnet(\g)$ of $\X$ must satisfy \cite{DBLP:journals/tit/GottliebKK14+colt}
\beqn
\label{eq:Ng}
|\gnet(\g)| \leq 
\Ng,
\eeqn
so the compression size $\k(\g)$ computed by \algname\ while using the margin $\g$ satisfies $\P[\k(\g) \leq \Ng] = 1$.
Hence, the first term in \eqref{split_miss} is bounded by
\begin{align*}
\label{eq:P_Q_sum}
&\P_{S_n}[ Q_n(\a(\g),\k(\g)) > R^* + \eps \;\wedge\; \missmass_\g(S_n) \leq \frac{\eps}{10}] \\
&\leq 
\sum_{d=1}^{\Ng} \P_{S_n}
\Big[
Q_n(\a(\g),\k(\g)) > R^* + \eps
\;\wedge\; \missmass_\g(S_n) \leq \frac{\eps}{10}
\;\wedge\; \k(\g)=d
\Big].
\end{align*}
Thus, it suffices to bound each term in the right-hand sum separately. We do so in the following lemma.

\begin{lemma} \label{lem:boundqd} Fix $\eps > 0$ and let $\g = \g(\epsilon) = \beta(\epsilon/10)/4$ as in \eqref{gammaeps}. Under the same conditions as \thmref{comp-consist}, there exists an $n_0$ such that for all $n \geq n_0$, and for all $d \in [\Ng]$,
\[
p_d := \P_{S_n}
\Big[
Q_n(\a(\g),\k(\g)) > R^* + \eps
\;\wedge\; \missmass_\g(S_n) \leq \frac{\eps}{10}
\;\wedge\; \k(\g)=d] \leq e^{-n\epsilon^2/32}.
\]
\end{lemma}
Applying \lemref{boundqd} and summing over all $1\leq d\leq \Ng$, we have that the first term in (\ref{eq:split_miss}) satisfies
\beqn
\label{eq:first_term}
\P_{S_n}[Q_n(\a(\g),\k(\g)) > R^* + \eps
\; \wedge\;  \missmass_\g(S_n) \leq \frac{\eps}{10}]
\;\leq\; \sum_{d=1}^{\Ng} p_d \;\leq\; \Ng \cdot e^{ -\frac{n \eps^2}{32} } = C_1(\gamma,\eps)e^{ -\frac{n \eps^2}{32} }.
\eeqn
Plugging \eqref{first_term} and \eqref{Psi_n} into \eqref{split_miss}, we get that (\ref{eq:termII_bound_fixed_g}) holds, which completes the proof of \thmref{comp-consist}. The proof of \lemref{boundqd} follows.
\begin{proof}[Proof of \lemref{boundqd}]
Let $\bm i = \bm i(\g) \in I_{n,d}$ be the set of indices of samples from $S_n= (X_i,Y_i)_{i\in[n]}$ selected by the algorithm for the construction of $\tS_n(\g)$, so
\beq
 \bm X  \equiv {\bm X}(\bm i)   =  \{X_{i_1}, \dots, X_{i_d} \}.
\eeq
By construction, the algorithm guarantees that $\bm X$ is a $\g$-net of $S_n$. 
This $\g$-net induces the Voronoi partition $\Vor(\bm X) = \{V_1,\dots, V_{d}\}$ of $\X$, where
\beqn
\label{eq:voronoi}
V_j = \{ x\in\X : X_{\nn}(x,\bm X) = X_{i_j}\},
\qquad j\in[d].
\eeqn
Let ${\bm Y}^*(\bm i) \in \Y^d$ be the true majority-vote labels with respect to the restricted partition $\Vor(\bm X)\cap\UB_{2\g}(\bm X)$,
\beqn
\label{eq:Y_star}
({\bm Y}^*)_j = 
y^*(V_j \cap \UB_{2\g}(\bm X)),
\qquad j\in[d].
\eeqn
We pair ${\bm X}(\bm i)$ with the labels $\bm Y^*(\bm i)$ to obtain the labeled set
\beqn
\label{eq:Strue}
S_n(\bm i, *) := S_n(\bm i, {\bm Y}^*(\bm i)) = 
(\bm X,\bm Y^*(\bm i))
\in (\X\times\Y)^d.
\eeqn
Note that $S_n(\bm i, *)$ is completely determined by $\bm X$ and does not depend on the rest of $S_n$. 

The induced 1-NN classifier $h_{S_n(\bm i, *)}(x)$ can be written as $h_{\SP,M}^*(x)=y^*(I_\SP(x) \cap M)$ with $\SP = \Vor(\bm X)$ and $M=\UB_{2\g}(\bm X)$ (see \eqref{Strue2} for the definition of $h_{\SP,M}^*$). We now show that 
\beqn
\label{eq:err_R}
\missmass_\g(S_n) \leq \frac{\eps}{10}
\quad \Rightarrow \quad
\err(h_{S_n(\bm i, *)})\leq R^* + \eps/2,
\eeqn
by showing that under the assumption $\missmass_\g(S_n) \leq \frac{\eps}{10}$, the conditions of \lemref{richness} hold for $\Vor,M$ as defined above. For this purpose, we bound the diameter of the partition $\Vor(\bm X) \cap \UB_{2\g}(\bm X)$ and the measure of the missing mass $\missmass_{2\g}(\bm X)$ under the assumption.

To bound the diameter of the partition $\Vor(\bm X) \cap \UB_{2\g}(\bm X)$, let $x \in V_j \cap  \UB_{2\g}(\bm X)$. Then $\rho(x,x_j) = \min_{i \in \bm i} \rho(x,x_i)$ and, since $x \in \UB_{2\g}(\bm X)$, $\min_{i \in \bm i} \rho(x,x_i) \leq 2\gamma$. Therefore
\beq
\diam(\SP \cap M) = \diam(\Vor(\bm X) \cap \UB_{2\g}(\bm X) )\leq 4\g.
\eeq
To bound $\missmass_{2\g}(\bm X)$ under the assumption $\missmass_\g(S_n) \leq \frac{\eps}{10}$, observe that for all $z \in \UB_{\g}(S_n)$, there is some $i \in [n]$ such that $z \in B_\g(x_i)$. and for this $i$ there is some $j \in \bm i$ such that $x_i \in B_\g(x_j)$, since $\bm X$ is a $\g$-net of $S_n$. Therefore $z \in B_{2\g}(x_j)$. Thus $z \in \UB_{2\g}(\bm X)$. It follows that $\UB_\g(S_n) \subseteq \UB_{2\g}(\bm X)$, thus $\missmass_{2\g}(\bm X) \leq \missmass_\g(S_n)$. Under the assumption, we thus have $\missmass_{2\g}(\bm X) \leq \frac{\eps}{10}$.
Hence, by the choice of $\g=\g(\eps)$ in the statement of the lemma, \lemref{richness} implies \eqref{err_R}.

To bound $Q_n(\a(\g), \k(\g))$, we consider the relationship between the hypothetical true majority-vote classifier $h_{S_n(\bm i, *)}$ and the actual classifier returned by the algorithm, $h_{S_n(\bm i, \bm \tbY)}$. Note that 
\beq
\a(\g) 
= \serr(h_{S_n(\bm i, \bm \tbY)}, S_n)
= \min_{\bm Y \in\Y^d}\serr(h_{S_n(\bm i, \bm Y)}, S_n)
\leq  \serr(h_{S_n(\bm i, *)}, S_n),
\eeq
and thus, from the monotonicity Property 2 of $Q$, 
\beqn
Q_n(\a(\g), \k(\g))
\leq
\label{eq:Q_Q}
Q_n(\serr(h_{S_n(\bm i, *)}, S_n), d).
\eeqn
Combining (\ref{eq:err_R}) and (\ref{eq:Q_Q}) we have that
\beq
&&
\left\{Q_n(\a(\g),\k(\g))
>
R^* + \eps
\;\;\wedge\;\;
\missmass_\g(S_n) \leq \frac{\eps}{10}
\;\;\wedge\;\;
\k(\g)=d
\right\}
\\
&& \qquad \;\Rightarrow\;
\left\{ Q_n(\serr(h_{S_n(\bm i, *)}, S_n), d) > \err(h_{S_n(\bm i, *)}) + \frac{\eps}{2}  \;\;\wedge\;\; |\bm i|=d \right\}.
\eeq
Hence, for all $d \leq \Ng$
\beqn
\nonumber
p_d &\leq& \P_{S_n}\left[
Q_n( \serr(h_{S_n(\bm i, *)}, S_n),d )
>
\err(h_{S_n(\bm i, *)}) + \frac{\eps}{2}
\;\wedge\;
|\bm i| = d
\right]
\\
\nonumber
&\leq& 
\P_{S_n}\left[
 \exists \bm i \in I_{n,d},
 : Q_n( \serr(h_{S_n(\bm i, *)}, S_n),d ) 
>
\err(h_{S_n(\bm i, *)}) +  \frac{\eps}{2}
\right]
\\
\label{eq:sum-bound-lb}
 &\leq&
\sum_{\bm i \in I_{n,d}} 
\P_{S_n}\left[
Q_n( \serr(h_{S_n(\bm i, *)}, S_n),d ) 
>
\err(h_{S_n(\bm i, *)}) +  \frac{\eps}{2}
\right].
\eeqn
Thus, it suffices to bound each term in the right-hand sum in (\ref{eq:sum-bound-lb}) separately.
Define 
\beq
r_{d,n} = \sup_{\alpha \in (0,1)} (Q_n(\alpha,d) - \alpha).
\eeq
From Property 3 of $Q$, we have that $\lim_{n \rightarrow \infty} r_{d,n} = 0$. 
We have 
\beq
Q_n(\serr(h_{S_n(\bm i, *)},S_n) , d) \leq \serr(h_{S_n(\bm i, *)}, S_n) + r_{d,n}.
\eeq
Let $\bm i' = \{1,\dots,n\}\setminus \bm i$ and note that
\beq
\serr(h_{S_n(\bm i, *)}, S_n) \leq \frac{n-d}{n}\cdot \serr(h_{S_n(\bm i, *)}, S_n(\bm i')) + \frac{d}{n}.
\eeq
Combining the two inequalities above, we get
\[
Q_n(\serr(h_{S_n(\bm i, *)},S_n) , d) \leq \serr(h_{S_n(\bm i, *)}, S_n(\bm i')) + \frac{d}{n} + r_{d,n}.
\]
Taking $n$ sufficiently large so that for all $d\leq \Ng$,
\beq
\frac{d}{n}  +   r_{d,n}
 \leq \frac{\eps}{4},
\eeq 
we have
\[
Q_n(\serr(h_{S_n(\bm i, *)},S_n) , d) \leq \serr(h_{S_n(\bm i, *)}, S_n(\bm i')) + \frac{\eps}{4} .
\]
Therefore, for such an $n$, 
\beq
&& \left\{ Q_n( \serr(h_{S_n(\bm i, *)}, S_n),d ) 
>
\err(h_{S_n(\bm i, *)}) +  \frac{\eps}{2} \right\}
\\
&& \qquad\qquad\Rightarrow\quad
\left\{\serr(h_{S_n(\bm i, *)}, S_n(\bm i')) 
>
\err(h_{S_n(\bm i, *)})  +  \frac{\eps}{4}\right\}.
\eeq
Thus, each term in (\ref{eq:sum-bound-lb}) is bounded above by
\beqn
\nonumber
&& \P_{S_n}\left[
\serr(h_{S_n(\bm i, *)}, S_n(\bm i')) 
>
\err(h_{S_n(\bm i, *)})   +  \frac{\eps}{4}
\right]
\\
\label{eq:exp-bounds}
& = &
\E_{S_n(\bm i)}
\left[
\P_{S_n(\bm i') \gn S_n(\bm i)}
\left[
\serr(h_{S_n(\bm i, *)}, S_n(\bm i')) 
>
\err(h_{S_n(\bm i, *)})  +  \frac{\eps}{4}
\right]
\right]
.
\eeqn
Since $\P_{S_n(\bm i') \gn S_n(\bm i)}$ is a product distribution, by Hoeffding's inequality,  we have that (\ref{eq:exp-bounds}) is bounded above by $e^{-2(n-d)(\frac{\eps}{4})^2}$. Therefore, from \eqref{sum-bound-lb}
\[
p_d \leq \sum_{\bm i \in I_{n,d}} e^{-2(n-d)(\frac{\eps}{4})^2}\leq |I_{n,d}| e^{-2(n-d)(\frac{\eps}{4})^2} \leq (n)^{d} e^{-2(n-d)(\frac{\eps}{4})^2} = e^{d\log(n)-2(n-d)(\frac{\eps}{4})^2}.
\]
Selecting $n$ large enough so that $d\log(n) \leq (n-d)(\frac{\eps}{4})^2$ and $d \leq n/4$, we get the statement of the lemma.
\end{proof}

\subsection{Comparison to previous strong Bayes-consistency results}
\label{ap:compare_to_aistats}
The strong Bayes-consistency result in \thmref{comp-consist} can be seen as an extension of
the general-purpose strong Bayes-consistency result in 
\cite[Theorem 21.2]{devroye2013probabilistic} for data-dependent partitioning rules, 
so as to overcome two additional entangled technical challenges:
(i) \algname\ is {\em adaptive}; and (ii) it relies on {\em compression-based} generalization bounds.
Indeed, since a 1-NN algorithm is a partition of $\X$ into majority-vote Voronoi cells, 
\cite[Theorem 21.2]{devroye2013probabilistic}
could be applied to show the strong Bayes-consistency of a {\em non-adaptive} version of \algname.
In our setting, being non-adaptive means the scale $\g_n$ used for a sample of size $n$ must be fixed in advance, as opposed to the optimized scale $\g^*$ chosen by \algname.

To
enable
adaptivity, one typically turns to the framework of Structural Risk Minimization \cite{DBLP:journals/tit/Shawe-TaylorBWA98}.
Such an approach was taken by \cite{DBLP:journals/tit/GottliebKK14+colt}, where an adaptive margin-regularized 1-NN algorithm (termed here \gkk) is analyzed based on the hierarchy of the class of $1/\g$-Lipschitz functions for $\g>0$.
These classes come with corresponding generalization bounds that can
be minimized to obtain an optimal scale constant $L^*_n = 2/\g_n^*$.
To minimize the bound, \gkk\ computes, for any choice of $\g>0$, a minimum vertex cover for removing from the sample the smallest number of points so as to ensure that it is $\g$-separated (meaning: no two points with conflicting labels are $\g$-close).
It was shown in \citet{kontorovich2014bayes} that this approach leads to a strongly Bayes-consistent 1-NN classifier.

Despite being adaptive, \gkk\ has some major limitations.
In particular, it is not computationally efficient for the multiclass problem \cite{kontorovich2014maximum}, the generalization bounds explicitly depend on the dimension of the space (see however \cite{gottlieb2016adaptive}), as well as the number of labels, and it suffers from the time and memory inefficiency of storing a large sub-sample.

To mitigate the last two limitations, one might consider the binary classification
algorithm implicit in \cite[Theorem 4]{gknips14} (termed here \gkn). 
\gkn\ minimizes the same
empirical error as \gkk, but the remaining portion of the sample $S$ is further compressed to a $\g$-net of size roughly $(\diam(S)/\g)^{\ddim}$.
\gkn\ then seeks the $\g$ which minimizes a compression-based generalization bound, qualitatively similar to \eqref{KSUbound}.
We conjecture that \gkn\ is Bayes-consistent, but were unable to prove it, despite numerous attempts. Due to
its reliance on
the minimum vertex cover, \gkn\ is also inefficient for the multiclass problem.

\algname\ overcomes all
of
these limitations by minimizing a compression-based generalization bound,
which amounts to
minimizing the empirical error via an efficient majority-vote rule.
In particular, it allows the efficient minimization of the bounds
that hold for infinite dimensional spaces and countable number of labels.

\section{Bayes-consistency proof of \algname\ on Preiss's construction}
\label{ap:preiss_construction}
In this section we prove \thmref{ksubayespreiss}.
We break the proof into three parts. 
After deriving the necessary properties of Preiss's construction in \secref{preiss_preliminaries}, we prove in \secref{preiss_true} that if \algname\ used the true majority-vote labels, it would be Bayes consistent. We then prove in \secref{preiss_empirical} that the same holds for empirical majority labels.

\subsection{Preliminaries}
\label{sec:preiss_preliminaries}
In this subsection, we characterize all possible forms a Voronoi cell can take in partitions of $\X$ that are induced by $\g$-nets.
As we show below, due to the structure of the construction and the packing and covering properties of $\g$-nets, each Voronoi cell can be one among only 4 specific types as stated in \lemref{preiss_gnet_aux}. These types will be used in subsequent sections to characterize the error incurred by any such majority-vote partition of $\X$.

\paragraph{Balls and subtrees.}
For any $z\in\fins\cup\infs$ of length $|z| = l\in\N\cup \{\infty\}$, we denote by $z_{1:k}=(z_1,\ldots,z_k)$ its prefix of length $k\leq l$ and write $z_{1:k}\prefix z$ for short.
By convention, $z_{1:0}=\emptyset$.
For any $l\geq 1$ and $z_{1:l}\in \fins$, we define the {\em subtree} rooted at $z_{1:l}$ by
\beq
T(z_{1:l}) = \{z' \in \fins \cup \infs : z_{1:l} \prefix z'\}.
\eeq
For a subtree $T(z_{1:l})$, we define $\Tu(z_{1:l})$ as $T(z_{1:l})$ augmented with its ancestor, $z_{1:l-1}$,
\beq
\Tu(z_{1:l}) = T(z_{1:l}) \cup \{z_{1:l-1}\}.
\eeq
We will make extensive use of the following properties (see Figure \ref{fig:preiss_construction}).

\begin{lemma}
\label{lem:gnet_balls}
 Let $z\in \infs = C$. For any $l\geq 1$,
\begin{itemize}
\item[(i)]  $\forall y,w \in \Tu(z_{1:l})$, $\dist(y,w) \leq {2 \seq{l}} = \seq{l-1}$.
\item[(ii)] $\forall y \notin \Tu(z_{1:l})$ and $w\in\Tu(z_{1:l})$,
$\dist(y,w) \geq \seq{l} + (\seq{l-1} -\seq{|z^w|})$.
\end{itemize}
In particular, taking $w = z$,
\begin{itemize}
\item[(iii)] 	$\forall y \notin \Tu(z_{1:l})$,  $\dist(y,z) \geq {\seq{l-1} + \seq{l} } > \seq{l-1}$.
\end{itemize}
Following from \textit{(i)} and \textit{(iii)}, the closed ball about any $z\in C$ satisfies $\cB_{\seq{l-1}}(z) = \Tu(z_{1:l})$.
\begin{itemize}
	\item[(iv)] 
		For any $l<\infty$ and any $y \in \fins$ such that $|y|=l$, $\cB_{\seq{l}^{}}(y) = \Tu(y)$.
\end{itemize}
\end{lemma}

\begin{proof} \textbf{\textit{(i)}} Let $w\in T(z_{1:l})$ and note that $|w| \geq l$ and $z_{1:l}\prefix w$.
For the case $y = z_{1:l-1}$,
$\dist(y, w) = \seq{l-1} - \seq{|w|} \leq \seq{l-1}$.
For the other cases $y\in T(z_{1:l})$, note that $\seq{|\ca{w}{y}|} \geq  \seq{l}$ and thus 
\beq
\dist(y, w) = 2\seq{|\ca{w}{y}|} - \seq{|w|} - \seq{|y|} \leq 2\seq{|\ca{w}{y}|} \leq  2\seq{l} = \seq{l-1}.
\eeq
To show \textbf{\textit{(ii)}}
note that for $y \notin \Tu(z_{1:l})$ and $w\in \Tu(z_{1:l})$, we have $|\ca{w}{y}|\leq l-1$.
There are two possible cases, $y\in T(z_{1:l-1})$ and $y\notin T(z_{1:l-1})$.
For the first case, since $y\notin \Tu(z_{1:l})$ we have $y\in T(z_{1:l-1})\setminus \Tu(z_{1:l})$ and thus $|\ca{w}{y}| = l-1$ and $|y|\geq l$.
Hence,
\beq
\dist(y, w) = 2\seq{l-1} - \seq{|w|} - \seq{|y|}
\geq 2\seq{l-1} - \seq{|w|} - \seq{l} = \seq{l} + (\seq{l-1} - \seq{|w|}).
\eeq
For the second case, $y\notin T(z_{1:l-1})$, we have $|\ca{w}{y}|<l-1$ and by definition 	 $|y| \geq |\ca{w}{y}|$, thus 
\beq
\dist(y, w)
\geq 2\seq{|\ca{w}{y}|} - \seq{|w|} - \seq{|\ca{w}{y}|}
\geq \seq{l-2} - \seq{|w|} > 
\seq{l} + (\seq{l+1} - \seq{|w|}).
\eeq
Part \textbf{\textit{(iii)}} readily follows from \textit{(ii)}.
To show \textbf{\textit{(iv)}} note that since $|y|=l$, for any $w\in T(y)$ we have 
$\dist(y, w) \leq \seq{l}$. Similarly, $\dist(y_{1:l-1},y) = \seq{l}$.
Thus, $\Tu(y)\subseteq \cB_{\seq{l}}(y)$.
In addition, for any $w\notin \Tu(y)$, part \textit{(ii)} implies $\dist(w,y)  \geq \seq{l} + (\seq{l-1} -\seq{l}) > \seq{l}$, so $\cB_{\seq{l}}(y)\subseteq\Tu(y)$.
\end{proof}

\paragraph{Bad Voronoi cells.} The following two lemmas show that the tree structure of the construction, combined with the packing and covering properties of $\g$-nets, imply that Voronoi cells with $V \cap C \neq \emptyset$ have a special form. Call a Voronoi cell $V$ \emph{pure} iff $V \cap C = \emptyset$, and \emph{impure} otherwise. 

\begin{lemma}
\label{lem:anc}
Let $\seq{k} = 2^{-k}$ and $\gnet_k$ be a $\seq{k}$-net with induced Voronoi cells $\V_k$. Let $V\in \V_k$ be an impure cell, $\anc\in V$ be its anchor, and $z \in V \cap C$. 
Then,
\begin{itemize}
\item[(I)] $\anc \in \Tu(z_{1:k+1})= \cB_{\seq{k}}(z)$ and there is no other anchor besides $\anc$ in $\Tu(z_{1:k+1})$.

\item[(II)] $T(z_{1:k+1})\subseteq V \subseteq \Tu(z_{1:k}) = \cB_{\seq{k-1}}(z)$.

\item[(III)] If $V \setminus \Tu(z_{1:k+1}) \neq \emptyset$ then $\anc = z_{1:k}$ and $T(z_{1:k})\subseteq V$.

\item[(IV)] If $V \setminus T(z_{1:k+1}) \neq \emptyset$ then $z_{1:k}\in V$.
\end{itemize}
\end{lemma}

\begin{proof}
We prove the above parts by their order.
\paragraph{Part \textit{(I)}.}
Since $V$ is a Voronoi cell of a $\seq{k}$-net, we must have $\dist(z,\anc) \leq \seq{k}$.
Hence, $\anc \in \cB_{\seq{k}}(z)$.
By Lemma \ref{lem:gnet_balls}, $\cB_{\seq{k}}(z)= \Tu(z_{1:k+1})$.
In addition, by the $\seq k$-net condition, any other anchor $\ancb\in\gnet_k$ must satisfy $\dist(\anc,\ancb) > \seq{k}$.
However, if $\ancb \in \Tu(z_{1:k+1})$, part \textit{(i)} of Lemma \ref{lem:gnet_balls} implies $\dist(\anc,\ancb) \leq \seq{k}$.
Hence there is no other anchor besides $\anc$ in $\Tu(z_{1:k+1})$.

\paragraph{Part \textit{(II)}.} 
To show $T(z_{1:k+1})\subseteq V$ note that by part (I) there is no other anchor in $\Tu(z_{1:k+1})$ besides $\anc$.
However, this implies that any other anchor $\ancb\neq \anc$ is too far to be the anchor of any $y\in T(z_{1:k+1})$, implying $y\in V$. 
To see this, note that since $\ancb \notin \Tu(z_{1:k+1})$, we have $|\ca{\ancb}{y}| \leq k$.
Similarly, $|\ca{\anc}{y}| \geq k$.
Consider first the case $|\ca{\anc}{y}| = k$.
Then it must be that $\anc = z_{1:k}$ and $\anc \prefix y$.
In addition, since $\ancb \notin \Tu(z_{1:k+1})$ we have $ \dist(\anc,\ancb) > 0$.
So by the definition of $\dist$,
\beq
\dist(y,\ancb) = \dist(y,\anc) + \dist(\anc,\ca{\anc}{\ancb}) + \dist(\ca{\anc}{\ancb},\ancb) =\dist(y,\anc) + \dist(\anc,\ancb) > \dist(y,\anc).
\eeq
Hence, $\ancb$ is too far from $y$ to be its anchor as claimed.
Next, consider the case $|\ca{\anc}{y}|>k$.
By definition, both $y,\anc \in T(\ca{\anc}{y})$, and
\beqn
\label{eq:part_II_1}
\dist(y,\anc) \leq \dist(y,\ca{\anc}{y}) + \dist(\ca{\anc}{y},\anc)
\leq \dist(y,\ca{\anc}{y}) + \seq{|\ca{\anc}{y}|}
\leq \dist(y,\ca{\anc}{y}) + \seq{k+1}.
\eeqn
On the other hand, since $\ancb \notin \Tu(z_{1:k+1})$ and $z_{1:k} \prefix \ca{\anc}{y} \prefix y$, 
\beq
\dist(y,\ancb) = \dist(y,\ca{\anc}{y}) + \dist(\ca{\anc}{y},\ancb).
\eeq
By part \textit{(ii)} of Lemma \ref{lem:gnet_balls} (with $w=\ca{\anc}{y}$ and $l= k+1$),
\beq
\dist(\ca{\anc}{y},\ancb) \geq \seq{k+1} + (\seq{k} -\seq{|\ca{\anc}{y}|}).
\eeq
Applying the fact that $|\ca{\anc}{y}|>k$, we thus get
\beqn
\label{eq:part_II_2}
\dist(y,\ancb) \geq \dist(y,\ca{\anc}{y}) + \seq{k+1} + (\seq{k} - \seq{k+1}) = \dist(y,\ca{\anc}{y}) + \seq{k}.
\eeqn
Hence, by (\ref{eq:part_II_1}) and (\ref{eq:part_II_2}),
$\dist(y,\ancb) - \dist(y,\anc) \geq \seq{k} - \seq{k+1} = \seq{k+1} > 0$.
Thus $\ancb$ is too far from $y$ to be its anchor in this case as well. 
It follows that $T(z_{1:k+1})\subseteq V$ as claimed.

To show $V \subseteq \Tu(z_{1:k})$, note that by part \textit{(I)}, $\anc \in  \Tu(z_{1:k+1})$ and thus $|\anc| \geq k$.
For any $y\notin \Tu(z_{1:k})$, by part \textit{(ii)} of Lemma \ref{lem:gnet_balls} (with $w=\anc$ and $l=k$),
\beq
\dist(y,\anc) \geq \seq{k} + (\seq{k-1} -  \seq{|\anc|}) \geq \seq{k-1} > \seq{k}.
\eeq
However, by the $\seq{k}$-net condition, for any $y\in V$ we must have $\dist(y,\anc)\leq\seq{k}$, thus $y\notin V$.

\paragraph{Part \textit{(III)}.} 
Let $y\in V \setminus \Tu(z_{1:k+1})\neq \emptyset$ and assume,
to get
a contradiction, that $\anc \neq z_{1:k}$.
By part (I) of this lemma, $\anc \in \Tu(z_{1:k+1}) = T(z_{1:k+1}) \cup \{z_{1:k}\}$, so it must be that $\anc \in T(z_{1:k+1})$ and $|\anc| \geq k+1$.
On the other hand, by part \textit{(II)}, $V \subseteq \Tu(z_{1:k})=T(z_{1:k}) \cup \{z_{1:k-1}\}$, so 
$y$ can be either $y=z_{1:k-1}$ or be in a different branch of $T(z_{1:k})$ which is not $T(z_{1:k+1})$, namely, $y$ is such that $z_{1:k} \prefix y$, $z_{1:k+1} \not\prefix y$ and $|y| \geq k+1$.
Consider first the case $y=z_{1:k-1}$. Then, $y\prefix \anc$ and
\beq
\dist(y,\anc) = \seq{|y|} -  \seq{|\anc|} = \seq{k-1} -  \seq{|\anc|} \geq \seq{k-1} -  \seq{k+1} > \seq{k}.
\eeq
However, this contradicts the assumption that $\anc$ is the anchor of $y$, since by the $\seq k$-net condition any $y\in V$ must satisfy $\dist(y,\anc)\leq\seq{k}$.
So $y \neq z_{1:k-1}$.

Next, we consider the other possible cases for $y$, that is, $z_{1:k} \prefix y$, $z_{1:k+1} \not\prefix y$ and $|y| \geq k+1$.
Since $\anc \in T(z_{1:k+1})$, we have $|\ca{\anc}{y}| = k$.
However, if either $|y| \geq k+2$ or $|{\anc}|\geq k+2$, then
\beq
\dist(y,\anc) = 2\seq{|\ca{\anc}{y}|} - \seq{|y|} - \seq{|{\anc}|}
\geq 2\seq{k} - \seq{k+1} - \seq{k+2} = \seq{k} + \seq{k+2} > \seq{k}.
\eeq
So it must be that $\anc = z_{1:k+1}$, $|y| = k+1$, and $y$ is the single point in $T(y)$ that belongs to $V$.
Let $w \in (T(y) \setminus \{y\}) \cap C$, so $w\notin V$ and there must be another anchor $\ancb\neq\anc$ such that $\\dist(w,\ancb)\leq\seq{k}$.
Since $w\in C$, part (I) of this Lemma implies $\ancb \in \Tu(w_{1:k+1}) = \Tu(y)$.
However, by part \textit{(iv)} of Lemma \ref{lem:gnet_balls}, $\Tu(y)= \cB_{\seq{k+1}}(y)$ and thus $\dist(y,\ancb)  \leq \seq{k+1}$.
On the other hand, $\dist(y,\anc) = 2\seq{k+1}$, 
which is in contradiction to the assumption that $\anc$ is the anchor of $y$.
This exhausts all cases and thus $\anc = z_{1:k}$.

We next show that $\anc = z_{1:k}$ implies $T(z_{1:k})\subseteq V$.
Indeed, 
for any $y\in T(z_{1:k})$,
$\dist(y,\anc) \leq \seq{k}$.
So by the $\seq k$-net condition there is no other anchor in $T(z_{1:k})$ besides $\anc$.
In addition, for any other anchor $\ancb\neq\anc$, we have $|\ca{\anc}{\ancb}| \leq k-1$ and thus $\dist(\anc,\ancb)>0$.
Hence, 
$\dist(y,\ancb) = \dist(y,\anc) + \dist(\anc,\ancb) >\dist(y,\anc)$,
implying $y\in V$ as claimed.

\paragraph{Part \textit{(IV)}.} 
Suppose $V \setminus T(z_{1:k+1}) \neq \emptyset$.
If $z_{1:k}\notin V$ then we must have $V \setminus \Tu(z_{1:k+1})\neq\emptyset$.
However, by part \textit{(III)}, $z_{1:k}= \anc \in V$, which is a contradiction.
\end{proof}

\begin{lemma}
\label{lem:preiss_gnet_aux} 
Let $\seq k= 2^{-k}$ and let $\gnet_k$ be a $\seq k$-net of $\X$ with induced Voronoi cells $\V_k$. Then, 
\begin{itemize}
	\item[(A)] Any impure Voronoi cell $V\in \V_k$ can be one of only four types:
\begin{itemize}
\item[-]	{Type $I_a$}: $V = T(z_{1:k})$ for some $z_{1:k}\in \fins$. 
\item[-]	{Type $I_b$}: $V = T(z_{1:k+1})$ for some $z_{1:k+1}\in \fins$. 
\item[-] {Type $II_a$}: $V = \Tu(z_{1:k})$ for some $z_{1:k}\in \fins$. 
\item[-] {Type $II_b$}: $V = \Tu(z_{1:k+1})$ for some $z_{1:k+1}\in \fins$. 
\end{itemize}
\item[(B)] For all sufficiently large $k\in\N$, cells of type $I_a,I_b$ have true-majority-vote 1, while those of type $II_a,II_b$ have true-majority-vote 0.
\end{itemize}
\end{lemma}

\begin{proof}[Proof of \lemref{preiss_gnet_aux}]
By part \textit{(II)} of Lemma \ref{lem:anc}, $T(z_{1:k+1}) \subseteq V$.
When equality holds, $V=T(z_{1:k+1})$, which is of type $I_a$.
By part (IV),
when $V\setminus T(z_{1:k+1})\neq\emptyset$, $V$ must contain $z_{1:k}$ and thus $\Tu(z_{1:k+1})\subseteq V$. When equality holds, $V=\Tu(z_{1:k+1})$ and we get type $II_a$.
When $V \setminus \Tu(z_{1:k+1}) \neq \emptyset$, part (III)
implies $T(z_{1:k})\subseteq V$.
When equality holds, $V = T(z_{1:k})$, which is of type $I_b$. 
By part \textit{(II)},
we have $V\subseteq\Tu(z_{1:k})$, so we are left with $V=\Tu(z_{1:k})$ which is
of type $II_b$.
Hence \text{\textit{(A)}} follows.

To show \text{\textit{(B)}},
consider first cells of type $I_a$. For these we have
\beq
\mu(T(z_{1:k}) \cap C) =
\Pr[ y \in \infs :  z_{1:k}\prefix y] = 
\alpha/(N_1\ldots N_k).
\eeq
Similarly,
\beqn
\label{eq:mu_T_z}
\mu(T(z_{1:k})) & = & \alpha/(N_1\ldots N_k) + \Pr[ y \in \fins : z_{1:k}\prefix y]
\\
&=& \alpha/(N_1\ldots N_k) + 
(1-\alpha) ( a_k + \sum_{j=1}^\infty a_{k+j} N_{k+1} \ldots N_{k+j}).
\eeqn
Thus,
\beq
\frac{\mu(T(z_{1:k}) \cap C)}{\mu(T(z_{1:k}))} &=& \frac{1}{1 + \frac{1-\alpha}{\alpha} N_1 \ldots N_k(a_k + \sum_{j=1}^\infty a_{k+j} N_{k+1} \ldots N_{k+j}) }
\\ &=& \frac{1}{1 + \frac{1-\alpha}{\alpha} \sum_{j=0}^\infty a_{k+j} N_{1} \ldots N_{k+j}}.
\eeq
By the first condition in \eqref{nkak}, $\sum_{k=1}^\infty a_k N_1 \ldots N_k = 1$.
Thus, we have $\sum_{j=0}^\infty a_{k+j} N_{1} \ldots N_{k+j} = \sum_{j=k}^\infty a_{j} N_{1} \ldots N_{j} \to 0$ when $k\to\infty$.
It follows that
\beqn
\label{eq:VCV_TCT}
\frac{\mu(T(z_{1:k}) \cap C)}{\mu(T(z_{1:k}))} \xrightarrow{k\to\infty} 1.
\eeqn
Hence, for sufficiently large $k$, cells of type $I_a$ have true-majority-vote $1$.
An identical
argument shows that
cells of
type $I_b$ have true-majority-vote $1$ as well.

For cells of type $II_a$,
\beq
\mu(\Tu(z_{1:k}) \cap C) = \mu(T(z_{1:k}) \cap C) =
\Pr[ y \in \infs : z_{1:k}\prefix y] = 
\alpha/(N_1\ldots N_k),
\eeq
the same as for type $I_a$. However, since $z_{1:k} \in \Tu(z_{1:k+1})$, we have $\mu(\Tu(z_{1:k+1})) \geq (1-\alpha)a_{k}$.
Thus, by the second condition in \eqref{nkak},
\beq
\frac{\mu(\Tu(z_{1:k}) \cap C)}{\mu(\Tu(z_{1:k}))} &\leq&
\frac{\alpha}{(1-\alpha)a_k N_1 \ldots N_{k+1}} \xrightarrow{k\to\infty} 0.
\eeq
Thus, for sufficiently large $k$, cells of type $II_a$ have true-majority-vote $0$.
Same argument shows that cells of type $II_b$ have true-majority-vote $0$ as well.
\end{proof}

\subsection{Consistency with true majority vote labels}
\label{sec:preiss_true}
Consider a $\g$-net $\gnet \equiv \gnet(\g)$ of $\X$. Recall that $\SP(\gnet) = \{\sp_1,\ldots\}$ is the Voronoi partition induced by $\gnet$ and $I_{\SP}: \X \to {\SP}$ such that $I_\SP(x)$ is the unique $\sp\in\SP$ for which $x\in\sp$.
Define the true majority-vote classifier $h_{\gnet}:\X\to\Y$ given by 
\beqn
\label{eq:Strue3}
h_{\gnet}(x) = y^*(I_\SP(x)),
\eeqn
where $y^*(A)$ is the true majority vote label of $A\subseteq \X$ as given in \eqref{bar y(E)}.

\begin{theorem} 
\label{thm:gnet_cons}
Assume $(\X,\rho)$, $\Y$ and $\bmu$ as in \exref{preiss}. For $k \in \N$, let $\gnet_k$ be some $\seq k$-net over $\X$. The sequence $(h_{\gnet_{k}})_{k\in\N}$ of true majority-vote classifiers is Bayes consistent on $\bmu$:
\beq
\lim_{k \rightarrow \infty}\err(h_{\gnet_{k}}) = 0.
\eeq
\end{theorem}

\begin{proof}[Proof of Theorem \ref{thm:gnet_cons}]
To bound the error of $h_{\gnet_k}$, first note that $h_{\gnet_k}$ may incur an error only from impure Voronoi cells.
Indeed, any pure $V$ contains only points from $\X\setminus C$, which all have label $0$ which is in agreement with the true majority-vote of $V$ as required.

As for impure cells, \lemref{preiss_gnet_aux} implies that these can be only of types $I_a,I_b,II_a,II_b$, 
and for all sufficiently large $k\in\N$,
all cells of type $I_a,I_b$ have true majority-vote 1, while all those of type $II_a,II_b$ have true majority-vote 0.

Due to their true majority-vote of $1$, each cell $V$ of type $I_a$ incurs an error
\beq
\mu( V \setminus C) = \mu(T(z_{1:k})\setminus C) = \Pr[y\in \fins: z_{1:k}\prefix y] = a_k + \sum_{j=1}^\infty a_{k+j} N_{k+1}\ldots N_{k+j}.
\eeq
There are at most $N_1 \ldots N_k$ cells of type $I_a$. Denoting by $\VIa_k \subseteq \V_k$ the set of all cells of type $I_a$, we have
\beq
\mu(\VIa_k \setminus C) \leq N_1 \ldots N_k(a_k + \sum_{j=1}^\infty a_{k+j} N_{k+1}\ldots N_{k+j}) = 
\sum_{j=k}^\infty a_{j} N_{1}\ldots N_{j} \xrightarrow{k\to\infty} 0.
\eeq
Similarly,
\beq
\mu(\VIb_k\setminus C) 
\leq 
\sum_{j=k+1}^\infty a_{j} N_{1}\ldots N_{j} \xrightarrow{k\to\infty} 0.
\eeq

As for cells of type $II_a$, due to their true majority-vote of $0$, each such cell incurs an error
\beq
\mu( V \cap C) = \mu(\Tu(z_{1:k}) \cap C) = \mu(T(z_{1:k})\cap C) =
\Pr[y\in \infs: z_{1:k}\prefix y] = \alpha/ (N_{1}\ldots N_{k}).
\eeq
Since for cells of type $II_a$ we have $z_{1:k-1}\in V$, there are at most $N_1\ldots N_{k-1}$ such cells in total, thus
\beq
\mu(\VIIa_k \cap C) \leq \frac{\alpha N_1 \ldots N_{k-1}}{N_1 \ldots N_{k}}= 
\frac{\alpha}{N_k}\xrightarrow{k\to\infty} 0,
\eeq
where the limit above follows from the last property of the construction in \eqref{nkak}.
Similarly, the total error incurred by cells of type $II_b$ is
\beq
\mu(\VIIb_k \cap C) \leq \frac{\alpha N_1 \ldots N_{k}}{N_1 \ldots N_{k+1}}= 
\frac{\alpha}{N_{k+1}}\xrightarrow{k\to\infty} 0.
\eeq
Hence, we conclude that
\beq
\err(h_{\gnet_k}) = \mu(\VIa_k\setminus C) + \mu(\VIb_k\setminus C) + \mu(\VIIa_k\cap C) + \mu(\VIIb_k\cap C)\xrightarrow{k\to\infty} 0.
\eeq
\end{proof}

\subsection{Consistency with empirical majority vote labels}
\label{sec:preiss_empirical}
The following theorem proves the consistency of \algname\ on the Preiss construction. \thmref{ksubayespreiss} immediately follows from this result. 
\begin{theorem}
Assume $(\X,\rho)$, $\Y$, and $\bmu$ as in \exref{preiss}. \algname\ is weakly-Bayes-consistent on $\bmu$. 
\end{theorem}

The following lemmas allow us to reason about empirical $\seq{k}$-nets.

\begin{lemma}  \label{lem:empnet}
For $z \in Z_0$, denote $P_z := \{ y \in  Z_{\infty} : z \prec y\}$. Denote $A_k = \{z \in Z_0 : |z| \leq k+1\}$ and let $B\subseteq \X$ such that
\beqn
\label{eq:cond}
A_k \subseteq B \quad\text{and}\quad\forall z \in A_k, B \cap P_z \neq \emptyset.
\eeqn
Then any $\seq k$-net of $B$ is also a $\g_k$-net of $\X$.
\label{lem:T_in_V_emp}
\end{lemma}

\begin{proof}
For any $z \notin B$, we have $|z| > k+1$ and there is some $y \in P_{z_{1:k}} \cap B$.
Let $\gnet_k$ be a $\seq k$-net of $B$.
Following the same steps as in the proof of part (II) in \lemref{anc}, the Voronoi cell $V$ that includes $y$ has $T(z_{1:k+1}) \subseteq V$.
Hence $z\in V$. It follows that $\X$ is covered by $\gnet_k$ and thus $\gnet_k$ is a $\seq k$-net of $\X$.
\end{proof}

\begin{lemma}
\label{lem:gnet_finite_size} For any $\g_k>0$, any $\g_k$-net of $\X$ in \exref{preiss} is of finite size.
\end{lemma}

\begin{proof} 
Let $A_k,P_z$ be defined as in \lemref{empnet}.
First note that $\X = A_{k-1} \cup \{\bigcup_{z: |z|=k+1} T(z)\}$ and $|A_{k-1}|<\infty$.
For any $z\in\fins$ with $|z|=k+1$, there exists in the $\g_k$-net a Voronoi cell $V \supseteq T(z)$. This follows by noting that for any $y\in P_z$, the Voronoi cell $I_{\V}(y)$ has $I_{\V}(y) \cap C\neq\emptyset$, thus by part (II) of \lemref{anc}, $T(z) \subseteq I_{\V}(y)$.
Since there are $N_1\ldots N_{k+1}<\infty$ number of subtrees $T(z)$ with $|z| = k+1$, and also $|A_{k-1}|<\infty$, there is a finite number of Voronoi cells in the partition of $\X$.
\end{proof}

\begin{proof}[Proof of \thmref{ksupreissemp}]
Let $X_n$ be the set of points in the sample $S_n$. Recall that for any $\g > 0$, $S_n(\g)$ is a $\g$-net of $X_n$ labeled by the empirical majority vote in $S_n$. Let $X_n(\g_k)$ be the set of points in $S_n(\g_k)$. 
Let $h_{S'_n}$ be the hypothesis returned by \algname\ when run with $\delta_n$ that satisfies Property 3 for $Q$, and recall that $S'_n = S_n(\g^*_n)$, where $\g^*_n$ is the scale selected by \algname\ for $S_n$.  

We show that for any $\epsilon,\delta \in (0,1)$ there exists an $n_0$ such that for any $n > n_0$, 
\[
\P[\err(h_{S'_n}) \leq \epsilon] \geq 1-\delta.
\]
This implies $\lim_{n\rightarrow \infty}\P[\err(h_{S'_n}) \leq \epsilon] = 1$, as required for weak consistency.

For each $k$, let $\gnet_k$ be some $\seq k$-net over $\X$. From \thmref{gnet_cons}, we have $\lim_{k \rightarrow \infty} \err(h_{\gnet_k}) = 0$. Let $k$ such that for all $\g_k$-nets $\gnet$ of $\X$, $\err(h_{\gnet}) \leq \epsilon$. Such a $k$ exists, since there is a finite number of possible impure cells in a $\g_k$-net over $\X$, and the partitioning of pure cells does not affect the error of $h_{\gnet}$.

\newcommand{\vk}{\mathbb{V}_k}

Let $A_k,P_z$ be defined as in \lemref{empnet}. 
From \exref{preiss} it can be verified that $|A_k| < \infty$, $\mu(A_k)>0$ and $\min_{z \in A_k}\mu(P_z) > 0$. Thus there exists an integer $n'$ such that for any $n > n'$, with a probability at least $1-\delta$, The condition \eqref{cond} in \lemref{empnet} holds for $B = X_n$. Thus in this case, $X_n(\g_k)$ is a $g_k$-net of $\X$. Therefore all the impure Voronoi cells $V$ induced by $X_n(\seq k)$, are of one of the types listed in \lemref{preiss_gnet_aux}. Denote this set of Voronoi cells by $\vk$.

Let $\beta = \min_{V\in \vk}(|\mu(V \cap C) - \mu(V \setminus C)|)$. From part (B) of \lemref{preiss_gnet_aux} we have that for sufficiently large $k$, $\beta > 0$. Assume w.l.o.g~that the selected $k$ satisfies this.

Invoking Property 3 of $Q$,
let $n$ be large enough such that \beqn \label{eq:prop3}
\sup_{\alpha \in [0,1],m \leq M_k} Q(n, \a, m, \delta_n)-\a \leq \epsilon,
\eeqn
where $M_k<\infty$ is the maximal size of a $\seq k$-net on $\X$ (such an $M_k$ exists by \lemref{gnet_finite_size}) and such that with a probability at least $1-\delta$, the following conditions all hold:
\begin{enumerate}
\item $X_n$ satisfies \eqref{cond}\label{cond1} with $X_n = B$.
\item The empirical majority vote of $S_n$ in each $V \in \vk$ is equal to the true majority vote in $V$.\label{cond2}
\item The following holds: \label{cond3}
\beq
\sum_{V \in \vk} \big|\frac{|V \cap C \cap X_n|}{|X_n|} - \mu(V \cap C)\big| +
\big|\frac{|V \setminus C \cap X_n|}{|X_n|} - \mu(V \setminus C)\big| \leq \epsilon.
\eeq 
\item $\err(h_{S'_n}) \leq Q(n, \serr(h_{S'_n},S_n), 2|S'_n|, \delta_n)$. \label{cond4}
\end{enumerate}
Conditions \ref{cond2},\ref{cond3} can be satisfied for large enough $n$ because $\beta > 0$ and $\vk$ is finite. 
Condition \ref{cond4} can be satisfied because of Property 1 of $Q$, by setting $n$ such that $\delta_n \ll \delta$. 

Assume that the conditions above hold. Then by condition \ref{cond1}, all impure Voronoi cells in $X_n(\g_k)$ are in $\vk$. For pure cells, clearly both the true and the empirical majority votes are zero. For cells in $\vk$, by condition \ref{cond2}, the true majority vote and the empirical
majority vote are equal. Therefore, $h_{S_n(\g_k)} = h_{X_n(\g_k)}$. Moreover,
from condition \ref{cond3}, 
$\serr(h_{S_n(\seq k)},S_n) \leq \err(h_{S_n(\seq k)}) + \epsilon$.  
Therefore, 
\beqn \label{eq:snxn}
\serr(h_{S_n(\g_k)},S_n) \leq \err(h_{X_n(\g_k)}) +\epsilon.
\eeqn
Now, \algname\ selects 
\[
\seq n^* = \argmin_\g Q(n, \serr(h_{S_n(\g)},S_n), 2|S_n(\g)|, \delta_n),
\]
and sets $S'_n = S_n(\seq n^*)$. Thus we have
\begin{align*}
\err(h_{S'_n}) &\leq Q(n, \serr(h_{S_n(\seq n^*)},S_n), 2|S_n(\seq n^*)|, \delta_n)  \\
&\leq Q(n, \serr(h_{S_n(\g_k)},S_n), 2|S_n(\seq k)|, \delta_n) \\
&\leq \serr(h_{S_n(\g_k)}, S_n) + \epsilon.
\end{align*}
The first inequality follows from condition \ref{cond4}. 
In the last inequality we used \eqref{prop3}. Combining this with \eqref{snxn} gives that with a probability at least $1-\delta$, 
\[
\err(h_{S'_n}) \leq \err(h_{X_n(\g_k)}) +2\epsilon.
\]
Moreover, since \eqref{cond} holds, $X_n(\g_k)$ is a $\g_k$-net of $\X$. Therefore $\err(h_{X_n(\g_k)}) \leq \epsilon$.
It follows that with a probability at least $1-\delta$, for all large enough $n$,
\[
\P[\err(h_{S'_n}) \leq 3\epsilon] \geq 1-\delta.
\]
Substituting $3\epsilon$ by $\epsilon$, this proves weak consistency of \algname\ on \exref{preiss}.
\end{proof}

\bibliographystyle{plain}
\bibliography{consist_nips2017}

\begin{thebibliography}{10}

\bibitem{MR2327897}
Christophe Abraham, G\'erard Biau, and Beno\^{\i}t Cadre.
\newblock On the kernel rule for function classification.
\newblock {\em Ann. Inst. Statist. Math.}, 58(3):619--633, 2006.

\bibitem{berend2012missing}
Daniel Berend and Aryeh Kontorovich.
\newblock The missing mass problem.
\newblock {\em Statistics \& Probability Letters}, 82(6):1102--1110, 2012.

\bibitem{berend2013concentration}
Daniel Berend and Aryeh Kontorovich.
\newblock On the concentration of the missing mass.
\newblock {\em Electronic Communications in Probability}, 18(3):1--7, 2013.

\bibitem{BKL06}
Alina Beygelzimer, Sham Kakade, and John Langford.
\newblock Cover trees for nearest neighbor.
\newblock In {\em ICML '06: Proceedings of the 23rd international conference on
  Machine learning}, pages 97--104, New York, NY, USA, 2006. ACM.

\bibitem{MR2235289}
G\'erard Biau, Florentina Bunea, and Marten~H. Wegkamp.
\newblock Functional classification in {H}ilbert spaces.
\newblock {\em IEEE Trans. Inform. Theory}, 51(6):2163--2172, 2005.

\bibitem{MR2654492}
G\'erard Biau, Fr\'ed\'eric C\'erou, and Arnaud Guyader.
\newblock Rates of convergence of the functional {$k$}-nearest neighbor
  estimate.
\newblock {\em IEEE Trans. Inform. Theory}, 56(4):2034--2040, 2010.

\bibitem{MR2267655}
V.~I. Bogachev.
\newblock {\em Measure theory. {V}ol. {I}, {II}}.
\newblock Springer-Verlag, Berlin, 2007.

\bibitem{DBLP:conf/cvpr/BoimanSI08}
Oren Boiman, Eli Shechtman, and Michal Irani.
\newblock In defense of nearest-neighbor based image classification.
\newblock In {\em CVPR}, 2008.

\bibitem{cerou2006nearest}
Fr{\'e}d{\'e}ric C{\'e}rou and Arnaud Guyader.
\newblock Nearest neighbor classification in infinite dimension.
\newblock {\em ESAIM: Probability and Statistics}, 10:340--355, 2006.

\bibitem{DBLP:journals/corr/ChaudhuriD14}
Kamalika Chaudhuri and Sanjoy Dasgupta.
\newblock Rates of convergence for nearest neighbor classification.
\newblock In {\em NIPS}, 2014.

\bibitem{CoverHart67}
Thomas~M. Cover and Peter~E. Hart.
\newblock Nearest neighbor pattern classification.
\newblock {\em IEEE Transactions on Information Theory}, 13:21--27, 1967.

\bibitem{DBLP:journals/pami/Devroye81}
Luc Devroye.
\newblock On the inequality of {C}over and {H}art in nearest neighbor
  discrimination.
\newblock {\em {IEEE} Trans. Pattern Anal. Mach. Intell.}, 3(1):75--78, 1981.

\bibitem{MR780746}
Luc Devroye and L{\'a}szl{\'o} Gy{\"o}rfi.
\newblock {\em Nonparametric density estimation: the $L{_{1}}$ view}.
\newblock Wiley Series in Probability and Mathematical Statistics: Tracts on
  Probability and Statistics. John Wiley \& Sons, Inc., New York, 1985.

\bibitem{devroye2013probabilistic}
Luc Devroye, L{\'a}szl{\'o} Gy{\"o}rfi, and G{\'a}bor Lugosi.
\newblock {\em A probabilistic theory of pattern recognition}, volume~31.
\newblock Springer Science \& Business Media, 2013.

\bibitem{MR0257325}
Herbert Federer.
\newblock {\em Geometric measure theory}.
\newblock Die Grundlehren der mathematischen Wissenschaften, Band 153.
  Springer-Verlag New York Inc., New York, 1969.

\bibitem{FH1989}
Evelyn Fix and Jr. Hodges, J.~L.
\newblock Discriminatory analysis. nonparametric discrimination: Consistency
  properties.
\newblock {\em International Statistical Review / Revue Internationale de
  Statistique}, 57(3):pp. 238--247, 1989.

\bibitem{floyd1995sample}
Sally Floyd and Manfred Warmuth.
\newblock Sample compression, learnability, and the {V}apnik-{C}hervonenkis
  dimension.
\newblock {\em Machine learning}, 21(3):269--304, 1995.

\bibitem{DBLP:journals/tit/GottliebKK14+colt}
Lee{-}Ad Gottlieb, Aryeh Kontorovich, and Robert Krauthgamer.
\newblock Efficient classification for metric data (extended abstract {COLT}
  2010).
\newblock {\em {IEEE} Transactions on Information Theory}, 60(9):5750--5759,
  2014.

\bibitem{gottlieb2016adaptive}
Lee-Ad Gottlieb, Aryeh Kontorovich, and Robert Krauthgamer.
\newblock Adaptive metric dimensionality reduction.
\newblock {\em Theoretical Computer Science}, 620:105--118, 2016.

\bibitem{gknips14}
Lee-Ad Gottlieb, Aryeh Kontorovich, and Pinhas Nisnevitch.
\newblock Near-optimal sample compression for nearest neighbors.
\newblock In {\em Neural Information Processing Systems (NIPS)}, 2014.

\bibitem{gkn-jmlr17+aistats}
Lee-Ad Gottlieb, Aryeh Kontorovich, and Pinhas Nisnevitch.
\newblock Nearly optimal classification for semimetrics (extended abstract
  {AISTATS} 2016).
\newblock {\em Journal of Machine Learning Research}, 2017.

\bibitem{graepel2005pac}
Thore Graepel, Ralf Herbrich, and John Shawe-Taylor.
\newblock {PAC}-{B}ayesian compression bounds on the prediction error of
  learning algorithms for classification.
\newblock {\em Machine Learning}, 59(1):55--76, 2005.

\bibitem{hall2005}
Peter Hall and Kee-Hoon Kang.
\newblock Bandwidth choice for nonparametric classification.
\newblock {\em Ann. Statist.}, 33(1):284--306, 02 2005.

\bibitem{HKSW2019}
Steve Hanneke, Aryeh Kontorovich, Sivan Sabato, and Roi Weiss.
\newblock Universal {B}ayes consistency in metric spaces.
\newblock {\em CoRR arXiv}, 2019.

\bibitem{Kallenberg02}
Olav Kallenberg.
\newblock {\em Foundations of modern probability. Second edition. Probability
  and its Applications}.
\newblock Springer-Verlag, 2002.

\bibitem{kontorovichsabatourner16}
Aryeh Kontorovich, Sivan Sabato, and Ruth Urner.
\newblock Active nearest-neighbor learning in metric spaces.
\newblock In {\em Advances in Neural Information Processing Systems}, pages
  856--864, 2016.

\bibitem{kontorovich2014bayes}
Aryeh Kontorovich and Roi Weiss.
\newblock A {B}ayes consistent 1-{NN} classifier.
\newblock In {\em Artificial Intelligence and Statistics (AISTATS 2015)}, 2014.

\bibitem{kontorovich2014maximum}
Aryeh Kontorovich and Roi Weiss.
\newblock Maximum margin multiclass nearest neighbors.
\newblock In {\em International Conference on Machine Learning (ICML 2014)},
  2014.

\bibitem{KL04}
Robert Krauthgamer and James~R. Lee.
\newblock Navigating nets: {S}imple algorithms for proximity search.
\newblock In {\em 15th Annual ACM-SIAM Symposium on Discrete Algorithms}, pages
  791--801, January 2004.

\bibitem{MR1366756}
Sanjeev~R. Kulkarni and Steven~E. Posner.
\newblock Rates of convergence of nearest neighbor estimation under arbitrary
  sampling.
\newblock {\em IEEE Trans. Inform. Theory}, 41(4):1028--1039, 1995.

\bibitem{warmuth86}
Nick Littlestone and Manfred~K. Warmuth.
\newblock Relating data compression and learnability.
\newblock unpublished, 1986.

\bibitem{MR0464128}
James~R. Munkres.
\newblock {\em Topology: a first course}.
\newblock Prentice-Hall, Inc., Englewood Cliffs, N.J., 1975.

\bibitem{MR1741506}
Vladimir Pestov.
\newblock On the geometry of similarity search: dimensionality curse and
  concentration of measure.
\newblock {\em Inform. Process. Lett.}, 73(1-2):47--51, 2000.

\bibitem{MR3061714}
Vladimir Pestov.
\newblock Is the {$k$}-{NN} classifier in high dimensions affected by the curse
  of dimensionality?
\newblock {\em Comput. Math. Appl.}, 65(10):1427--1437, 2013.

\bibitem{preiss1979invalid}
David Preiss.
\newblock Invalid {V}itali theorems.
\newblock {\em Abstracta. 7th Winter School on Abstract Analysis}, pages
  58--60, 1979.

\bibitem{MR609946}
David Preiss.
\newblock Gaussian measures and the density theorem.
\newblock {\em Comment. Math. Univ. Carolin.}, 22(1):181--193, 1981.

\bibitem{335893}
Demetri Psaltis, Robert~R. Snapp, and Santosh~S. Venkatesh.
\newblock On the finite sample performance of the nearest neighbor classifier.
\newblock {\em IEEE Transactions on Information Theory}, 40(3):820--837, 1994.

\bibitem{MR0385023}
Walter Rudin.
\newblock {\em Principles of mathematical analysis}.
\newblock McGraw-Hill Book Co., New York, third edition, 1976.
\newblock International Series in Pure and Applied Mathematics.

\bibitem{rudin87}
Walter Rudin.
\newblock {\em Real and Complex Analysis}.
\newblock McGraw-Hill, 1987.

\bibitem{samworth2012}
Richard~J. Samworth.
\newblock Optimal weighted nearest neighbour classifiers.
\newblock {\em Ann. Statist.}, 40(5):2733--2763, 10 2012.

\bibitem{shwartz2014understanding}
Shai Shalev-Shwartz and Shai Ben-David.
\newblock {\em Understanding Machine Learning: From Theory to Algorithms}.
\newblock Cambridge University Press, 2014.

\bibitem{DBLP:journals/tit/Shawe-TaylorBWA98}
John Shawe-Taylor, Peter~L. Bartlett, Robert~C. Williamson, and Martin Anthony.
\newblock Structural risk minimization over data-dependent hierarchies.
\newblock {\em IEEE Transactions on Information Theory}, 44(5):1926--1940,
  1998.

\bibitem{MR1635410}
Robert~R. Snapp and Santosh~S. Venkatesh.
\newblock Asymptotic expansions of the {$k$} nearest neighbor risk.
\newblock {\em Ann. Statist.}, 26(3):850--878, 1998.

\bibitem{stone1977}
Charles~J. Stone.
\newblock Consistent nonparametric regression.
\newblock {\em The Annals of Statistics}, 5(4):595--620, 1977.

\bibitem{MR1974687}
Jaroslav Ti\v{s}er.
\newblock Vitali covering theorem in {H}ilbert space.
\newblock {\em Trans. Amer. Math. Soc.}, 355(8):3277--3289, 2003.

\bibitem{DBLP:journals/jmlr/WeinbergerS09}
Kilian~Q. Weinberger and Lawrence~K. Saul.
\newblock Distance metric learning for large margin nearest neighbor
  classification.
\newblock {\em Journal of Machine Learning Research}, 10:207--244, 2009.

\bibitem{MR877849}
Lin~Cheng Zhao.
\newblock Exponential bounds of mean error for the nearest neighbor estimates
  of regression functions.
\newblock {\em J. Multivariate Anal.}, 21(1):168--178, 1987.

\end{thebibliography}

\end{document}